%% file: main.tex
\documentclass[11pt]{article}
\usepackage[preprint]{acl}
\usepackage{times}
\usepackage{latexsym}
\usepackage[T1]{fontenc}
\usepackage[utf8]{inputenc}
\usepackage{microtype}
\usepackage{inconsolata}
\usepackage{graphicx}
\usepackage{booktabs}
\usepackage{amsmath}
\usepackage{multirow}
\usepackage[table]{xcolor}
\usepackage{amssymb}
\usepackage[most]{tcolorbox}
\usepackage{stfloats}
\usepackage{amsthm}
\newtheorem{proposition}{Proposition}
\newtheorem{corollary}{Corollary}

\def \name{\textsc{LookBack}}
\newcommand{\ta}[1]{\textbf{#1}}
\newcommand{\tb}[1]{\underline{#1}}
\newcommand{\tc}[1]{\text{#1}}
\newcommand{\ca}[1]{\cellcolor{gray!20}\textbf{#1}}
\newcommand{\cb}[1]{\cellcolor{gray!20}\underline{#1}}
\newcommand{\cc}[1]{\cellcolor{gray!20}\text{#1}}

\title{\name{}: Where and How to Score \\ LVLM Responses via Visual Reference Usage}

\author{
    Beomsik Cho\textsuperscript{*} \quad
    Jinhyeong Kim\textsuperscript{*} \quad
    Dongseok Lee \quad
    Jaehyung Kim \\
    Yonsei University \\
    \texttt{\{bscho333, mjmps0726, ehdtjr1220, jaehyungk\}@yonsei.ac.kr}
}

\begin{document}
\maketitle

\begingroup
\renewcommand{\thefootnote}{*}
\footnotetext{Equal contribution.}
\endgroup

\input{Tex/Abstract}
\input{Tex/Introduction}
\input{Tex/Related}
\input{Tex/Motivation}
\input{Tex/Method}
\input{Tex/Experiments}
\input{Tex/Conclusion}

\bibliography{custom}
\input{Tex/Appendix}

\end{document}

%% file: Tex/Abstract.tex
\begin{abstract}
Large Vision-Language Models (LVLMs) integrate visual perception with language generation, enabling responses that span image understanding and complex reasoning.
However, LVLMs do not just inherit the text-level hallucinations; they also hallucinate against the image, producing fluent responses ungrounded in what they see.
This makes LVLM response scoring inherently harder, and our diagnostics show that existing confidence-based metrics adopted from LLMs are insufficient for LVLMs.
Specifically, removing the input image barely changes confidence-based selection, suggesting that output-space confidence primarily captures textual plausibility rather than agreement with the image.
To address this gap, we propose \name{}, a training-free LVLM response scoring method that augments token likelihood with visual lookback score, a lightweight measure of how strongly each response token refers to image tokens.
Across four benchmarks and three models, \name{} consistently improves Best-of-$N$ selection over existing baselines with negligible additional overhead.
\footnote{Code: \url{https://github.com/bscho333/LookBack}.}
\end{abstract}

%% file: Tex/Introduction.tex
\section{Introduction}

With the recent success of Large Language Models (LLMs) \citep{touvron2023llama,achiam2023gpt,team2023gemini}, Large Vision-Language Models (LVLMs) have integrated visual understanding with text generation, demonstrating strong capability from image perception to complex reasoning~\citep{dai2023instructblipgeneralpurposevisionlanguagemodels,liu2024improved,zhu2023minigpt,bai2025qwen2,zhu2025internvl3}.
Despite these advances, LVLMs inherit a fundamental challenge of autoregressive generation: they can produce responses that are fluent and plausible yet factually incorrect.
In multimodal settings, this challenge is further complicated by visual hallucination, where a response asserts objects, attributes, or relations that are not supported by the image~\citep{rohrbach2018chair,li2023pope,leng2024vcd,guan2024hallusionbench}. 
Consequently, reliably identifying which responses are faithful to the image becomes an important problem for LVLMs.

A common setting for improving reliability is \textit{Best-of-N} selection, originating from the LLM literature: multiple candidate responses are sampled, and the best one is selected~\citep{wang2023self,gui2024bonbon,snell2024scaling}. 
The core of Best-of-$N$ is the selection criterion that scores each candidate, and a prominent way is to train reward models or verifiers to assess response quality~\citep{cobbe2021training, lightman2024let}. 
This idea has recently been extended to LVLMs through multimodal reward models that evaluate the visual correctness of candidates~\citep{wang2025visualprm, chen2025vrprm}. 
While effective, these methods require additional models, task-specific supervision, or preference annotations, which limits their applicability.

\input{Figures/motivation_overall}

A natural auxiliary-free direction is the model's own output-space confidence, \textit{i.e.}, how plausible a response is under its output distribution. In the LLM literature, such confidence signals have proven a useful selection criterion that requires no external verifier~\citep{wang2024chain,kang2025scalable}.
This makes the model's own output distribution an appealing candidate for scoring LVLM responses, as it requires no auxiliary model.
However, LVLM response scoring differs from its text-only counterpart in one crucial respect: response quality depends on an additional source of evidence, \emph{the image}.
Through a diagnostic analysis, we find that confidence-based selection remains nearly as effective even when the confidence is computed without conditioning on the image.
This reveals a gap between image-conditioning and image-sensitivity:
the confidence score reflects a response's textual plausibility rather than its agreement with the image.
Since visual hallucinations arise from a mismatch between the response and the image, such an image-insensitive signal cannot reliably detect them.

To close this gap, we incorporate an attention-based measure of \emph{visual reference usage} that captures how strongly each generation step attends back to vision tokens.
Specifically, we propose \name{}, a training-free LVLM response scoring method that calibrates output-space confidence using visual lookback scores.
At the token level, \name{} combines each token's likelihood with its visual lookback score to form a \textit{lookback calibrated token score}. 
At the response level, it aggregates these token scores under a visual relevance distribution, giving greater weight to tokens with stronger visual reference usage.
This distinction is crucial because LVLM responses contain both visually diagnostic content words and generic fluency words; lookback calibrated token scoring captures how strongly each token prediction references the visual input, while response-level aggregation determines how much visually relevant words contribute to the \name{}.

As a result, \name{} favors responses whose high-confidence tokens are strongly tied to the visual input, without requiring external verifiers, additional training, or extra inference passes.
We evaluate \name{} in Best-of-$N$ selection across four benchmarks that require visual understanding and three representative LVLMs: LLaVA-1.5-7B~\citep{liu2024improved}, Qwen2.5-VL-7B~\citep{bai2025qwen2}, and InternVL3-8B~\citep{zhu2025internvl3}. Across these settings, \name{} achieves the highest model-wise average performance for all three LVLMs and generally improves over both linguistic- and vision-side baselines.

%% file: Figures/motivation_overall.tex
\begin{figure*}[t]
    \includegraphics[width=\linewidth]{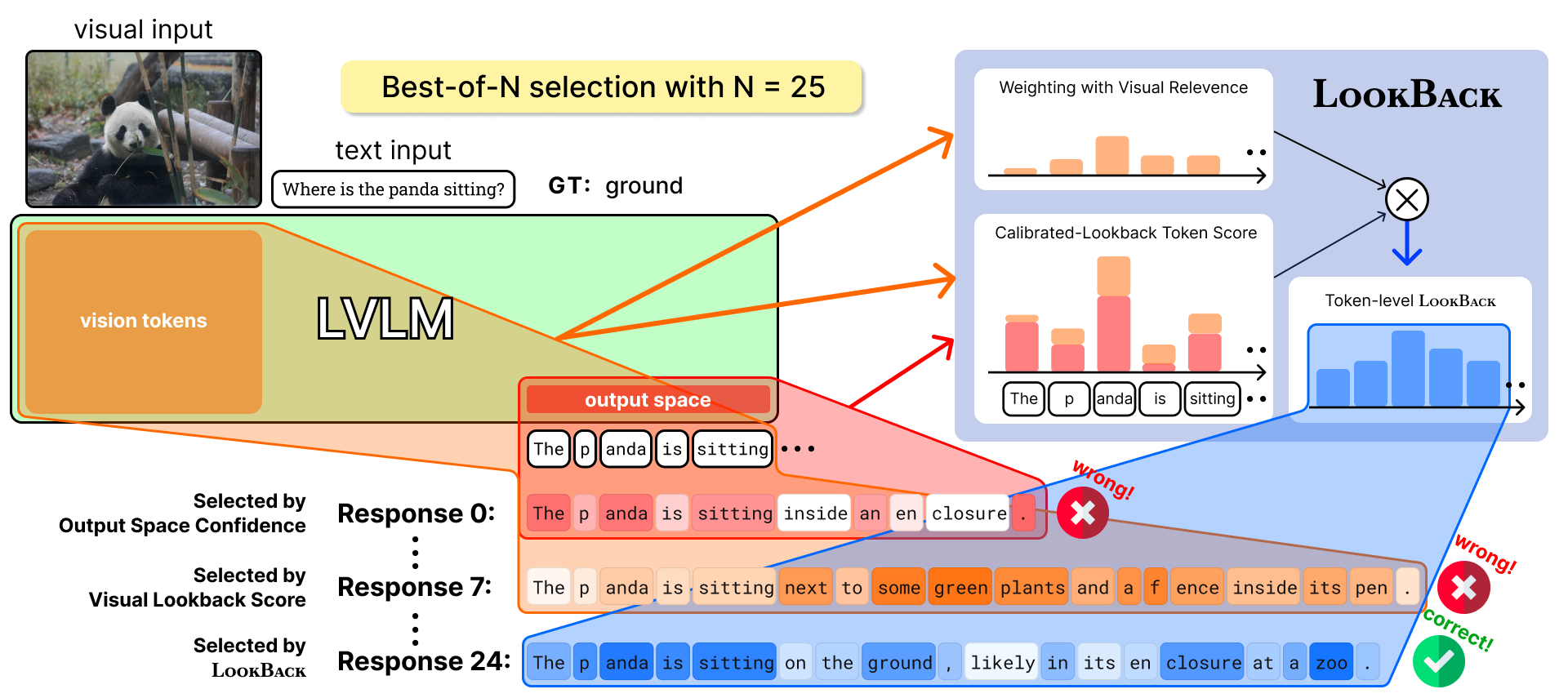}
    \caption{\textbf{Overview of \name{}: combining output-space confidence with visual reference usage.}
    Given an image and a question, an LVLM generates multiple candidate responses and our goal is to select the proper one.
Output-space confidence can assign high scores to fluent but visually unsupported tokens, while visual reference usage highlights tokens that directly look back to the image but does not by itself ensure response correctness.
\name{} (Ours) combines these two signals at the token level by calibrating token confidence with visual reference usage, and aggregates token scores at the response level with visual relevance weights.
    }
    \label{fig:motivation_overall}
\end{figure*}

%% file: Tex/Related.tex
\section{Related Works}

\paragraph{Output-space response scoring in LLMs.}
Stochastic decoding in LLMs yields multiple candidate responses that may follow different inference paths~\citep{kadavath2022language, qiu2024semantic}, making \textit{Best-of-$N$} selection a natural strategy to improve the performance and reliability~\citep{gui2024bonbon, snell2024scaling}.
Existing methods typically score candidates from the generated output alone, using answer consistency~\citep{wang2023self}, external reward models or verifiers~\citep{cobbe2021training, uesato2022solving, lightman2024let, wang2024math}, or training-free model-intrinsic signals~\citep{wang2024chain, kang2025scalable, lee2025training, gwak2025revisiting}.
However, the output-space score is less reliable in LVLMs, since fluent and confident responses can still contain visual claims unsupported by the image.
\name{} retains the lightweight, model-intrinsic nature of such scores while incorporating visual reference usage.

\paragraph{Response scoring in LVLMs.}
Unlike LLM response scoring, LVLM response scoring must evaluate both linguistic plausibility and visual faithfulness simultaneously.
Prior work has addressed this challenge using external multimodal evaluators or reward models, such as VLM-as-a-judge models, multimodal critics, and process reward models~\citep{lee2024prometheus, xiong2025llava, wang2025visualprm, chen2025vrprm}.
Other approaches estimate image--response alignment through cross-modal matching or calibrated visual constraints~\citep{hessel2021clipscore, zhou2024calibrated}, or analyze LVLM reliability through uncertainty, hallucination detection, and visual-grounding signals~\citep{leng2024vcd, park2026vauq}.
While these works emphasize the need for visual evidence, they often introduce external models, additional training, or extra scoring passes.
In contrast, \name{} scores responses using token likelihood and visual lookback score available from the LVLM's internal generation statistics, providing a training-free and efficient alternative for LVLM response selection.

\paragraph{Decoding-time hallucination mitigation.}
Prior decoding-time methods mitigate visual hallucination by modifying next token predictions through visual contrast, attention intervention, or visual guidance~\citep{leng2024vcd,liu2024pai,huo2024sid,cho2026revisit}. \name{} instead leaves generation unchanged and scores already generated responses. Since Best-of-$N$ setup is agnostic with decoding strategies, the scorer can complement these methods.

%% file: Tex/Motivation.tex
\section{Motivation}
\label{sec:motivation}

To identify a reliable scoring function for LVLM response, we begin with output-space confidence, as it has shown consistent effectiveness in LLM scoring across inference-time selection~\citep{kang2025scalable} and training-time reward~\citep{zhao2025learning}.
Our investigation asks \textit{whether this signal remains reliable when response quality depends not only on linguistic plausibility, but also on visual input.}

We instantiate this investigation in a Best-of-$N$ (BoN) response selection setting, where multiple candidate responses are generated for the same input and then ranked by a scoring function.
Our analysis covers both discriminative and generative LVLM responses on MS-COCO~\citep{lin2014mscoco} dataset: VQAv2~\citep{goyal2017vqav2} for visual question answering and CHAIR~\citep{rohrbach2018chair} for object hallucination evaluation.
For each benchmark, we randomly sample 1,000 instances and generate $N=25$ candidate responses with LLaVA-1.5-7B~\citep{liu2024improved} and Qwen2.5-VL-7B~\citep{bai2025qwen2}.
As a representative output-space confidence scorer, we adopt Self-Certainty (SC)~\citep{kang2025scalable}, a state-of-the-art BoN scoring method based on the token-averaged KL divergence between the model's predicted distribution and a uniform distribution.

\subsection{Is Output-Space Confidence Sensitive to Visual Input?}
\label{sec:motivation_1}

We first examine whether SC is affected by the visual input by comparing its score distributions under image-present and image-absent conditions.
For each generated response, we compute SC under two conditions: with the input image (SC w/ image) and without the input image (SC w/o image).
Figure~\ref{fig:motivation_output_dist} shows that across all model--benchmark combinations, the SC distributions under the two conditions are nearly identical in shape and mean, suggesting that SC scores are largely unaffected by the removal of the image at the population level.

\input{Figures/motivation_output_dist}

To further assess this weak image sensitivity at the selection level, we measure the \emph{top-1 agreement ratio}: the fraction of instances for which SC w/ image and SC w/o image select the same top-ranked candidate from $N=25$, directly reflecting whether the input image influences Best-of-$N$ selection under SC.
The solid bars in Figure~\ref{fig:motivation_agreement} report this agreement. 
If the image were a critical determinant of response quality, we would expect the two scorers to select different top-1 candidates.
However, we observe agreement ratios of 0.36--0.64, far above the random baseline of $1/N=0.04$, indicating that removing the image barely changes which response SC selects.
{
This trend persists when the image is replaced by a randomly sampled one, indicating that it stems from the scoring function itself rather than from the degenerate image-free input or from dataset-specific answer priors; see Appendix~\ref{appendix:random_image} for details.
}

These results reveal a gap between \emph{image-conditioning} and \emph{image-sensitivity}: although LVLMs are conditioned on the image during generation, their output-space confidence can remain largely unchanged even without the image.
Thus, output-space confidence alone cannot reliably indicate whether a response is grounded in the image.

\input{Figures/motivation_agreement}

\subsection{Visual Lookback Complements Token-Level Confidence}
\label{sec:motivation_2}

The weak image-sensitivity of output-space confidence raises a natural question: \textit{what model-internal signal can reflect whether each generation step consults the visual input?}

Unlike text-only LLM, LVLMs receive the image as explicit vision tokens in the input context.
Although visual information may be propagated implicitly through image-conditioned text representations and previous output states, the original vision tokens remain available in the causal context throughout decoding. 
Namely, each response-token prediction can therefore directly \emph{look back} to the visual input.
We quantify this behavior by the fraction of attention directed from each output token to the vision tokens.
We call this token-level quantity the \emph{visual lookback score} $A_t$, which serves as a lightweight proxy for \emph{visual reference usage}; its formal definition is given in Section~\ref{sec:method_1}.

\input{Figures/motivation_pos}

To examine whether this signal indeed captures a token-level visual reference, we analyze how it varies across different types of generated words.
If visual lookback score captures how much a generation step refers back to visual evidence, it should be higher for tokens that express visually referential content than for tokens that primarily serve grammatical or discourse functions. 

To this end, we conduct a token-level Part-of-Speech (POS) analysis on the same generated responses from Section~\ref{sec:motivation_1}.
Generated words are tagged with spaCy and grouped into two sets following prior works~\citep{chen2024halc, dong2025inter, min2025mitigating}: a \emph{visual set}, containing words more likely to express visually referential content---nouns, proper nouns, adjectives, and numerals---and a \emph{textual set}, containing words that primarily serve grammatical or discourse roles, such as auxiliaries, determiners, pronouns, and conjunctions.
Over the same word groups, we also analyze SC to test whether output-space confidence favors visually referential words despite its weak image sensitivity at the response-selection level.
For words split into multiple subtokens, we use the mean subtoken score as the word-level score, and independently z-score normalize each score dimension over all analyzed output tokens before word-level aggregation.

As shown in Figure~\ref{fig:motivation_pos}, across all model-benchmark combinations, words in the visual set exhibit above-average visual lookback score but below-average SC, whereas words in the textual set show the opposite pattern.
This opposing pattern suggests that visual lookback score captures a token-level tendency complementary to SC: SC is higher for linguistically predictable words, whereas visual lookback score is higher for words that are more visually referential.
A fine-grained per-POS breakdown in Figure~\ref{fig:appendix_pos} further shows that this trend is not driven by a single POS category, but broadly appears across content-oriented categories.

This token-level complementarity also appears at the response level.
The hatched bars in Figure~\ref{fig:motivation_agreement} show that SC and visual lookback score rarely select the same top-ranked response, with the lowest agreement ratio reaching 0.01, below the random baseline of $0.04$.
This indicates that the two signals induce substantially different response rankings.

Figure~\ref{fig:motivation_overall} illustrates this behavior qualitatively.
SC assigns high scores to fluent, linguistically predictable tokens including function words and assertion markers, whereas visual lookback score highlights object names, attributes, actions, and spatial relations that require direct consultation of the image.
This divergence is particularly pronounced when the model generates confident but visually unsupported assertions: confidence rewards linguistic fluency regardless of visual grounding, whereas visual lookback score reflects whether the prediction step directly attended to the visual input.
	
Overall, these observations suggest that output-space confidence and visual lookback score capture complementary aspects of LVLM generation.
Output-space confidence is useful for identifying linguistically likely responses, but it may remain high even when a response is driven primarily by language priors rather than visual evidence.
Visual lookback score captures whether generation steps attend back to the explicit vision tokens, but high visual attention alone does not guarantee correctness. 
Therefore, a visually grounded response score should combine token-level confidence with token-level visual lookback score, while placing greater emphasis on response positions where visual evidence is most relevant.

%% file: Figures/motivation_output_dist.tex
\begin{figure}[t]
    \includegraphics[width=\linewidth]{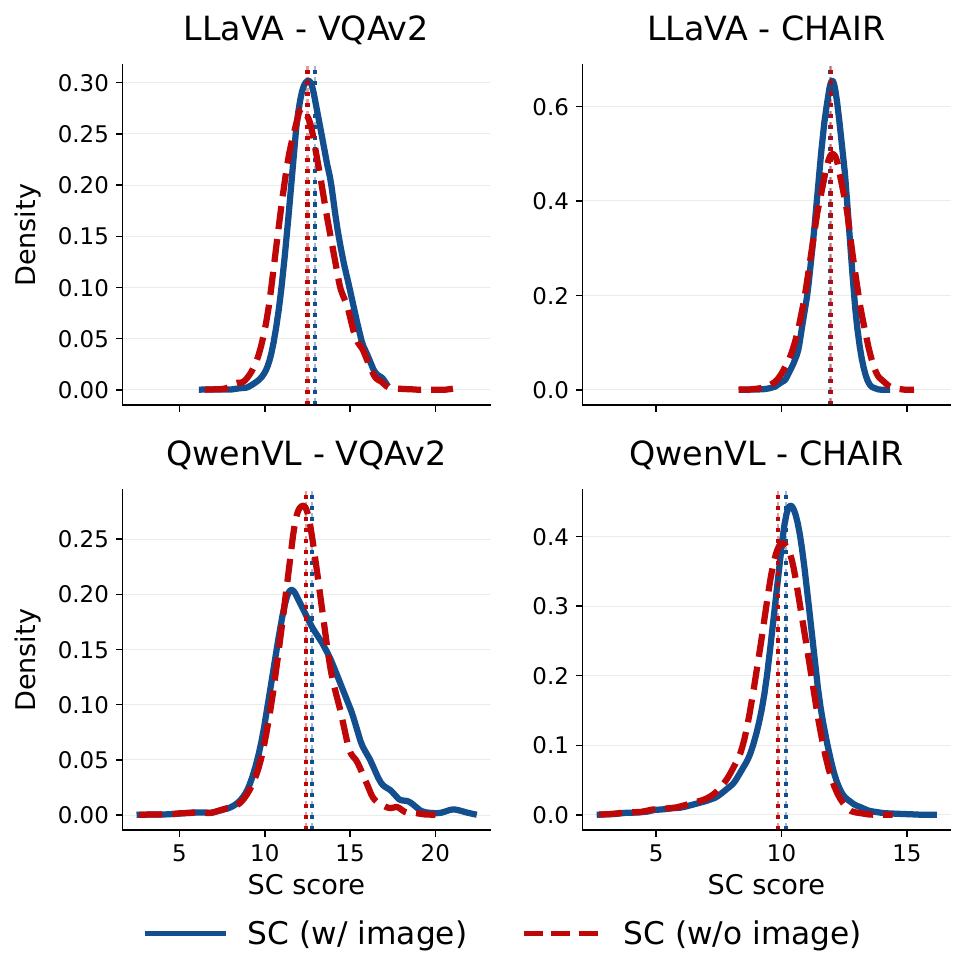}
    \caption{\textbf{SC score distributions remain similar with and without the image.}
    SC score distributions under image-present and image-removed conditions highly overlap, with closely aligned means across setups.
    }
    \label{fig:motivation_output_dist}
\end{figure}

%% file: Figures/motivation_agreement.tex
\begin{figure}[t]
    \includegraphics[width=\linewidth]{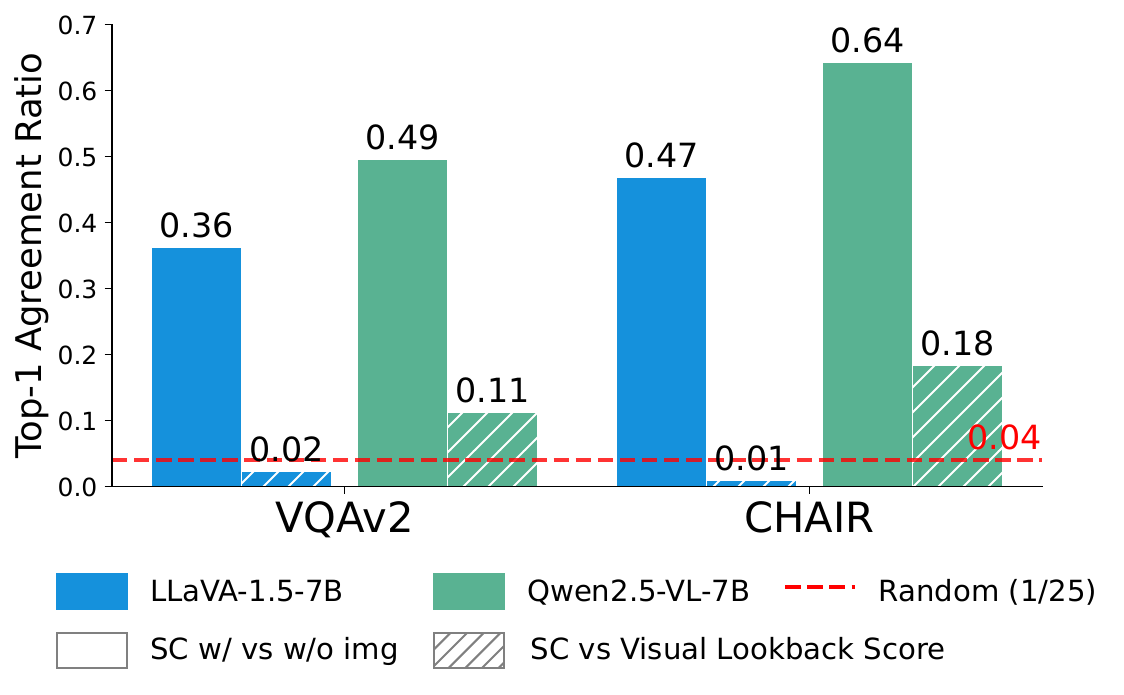}
    \caption{\textbf{Top-1 agreement with image-conditioned SC: image-removal consistency and visual-lookback complementarity.}
    We report the fraction of instances in which two scorers select the same top-ranked response from $N=25$ candidates.
    }
    \label{fig:motivation_agreement}
\end{figure}

%% file: Figures/motivation_pos.tex
\begin{figure}[t]
    \includegraphics[width=\linewidth]{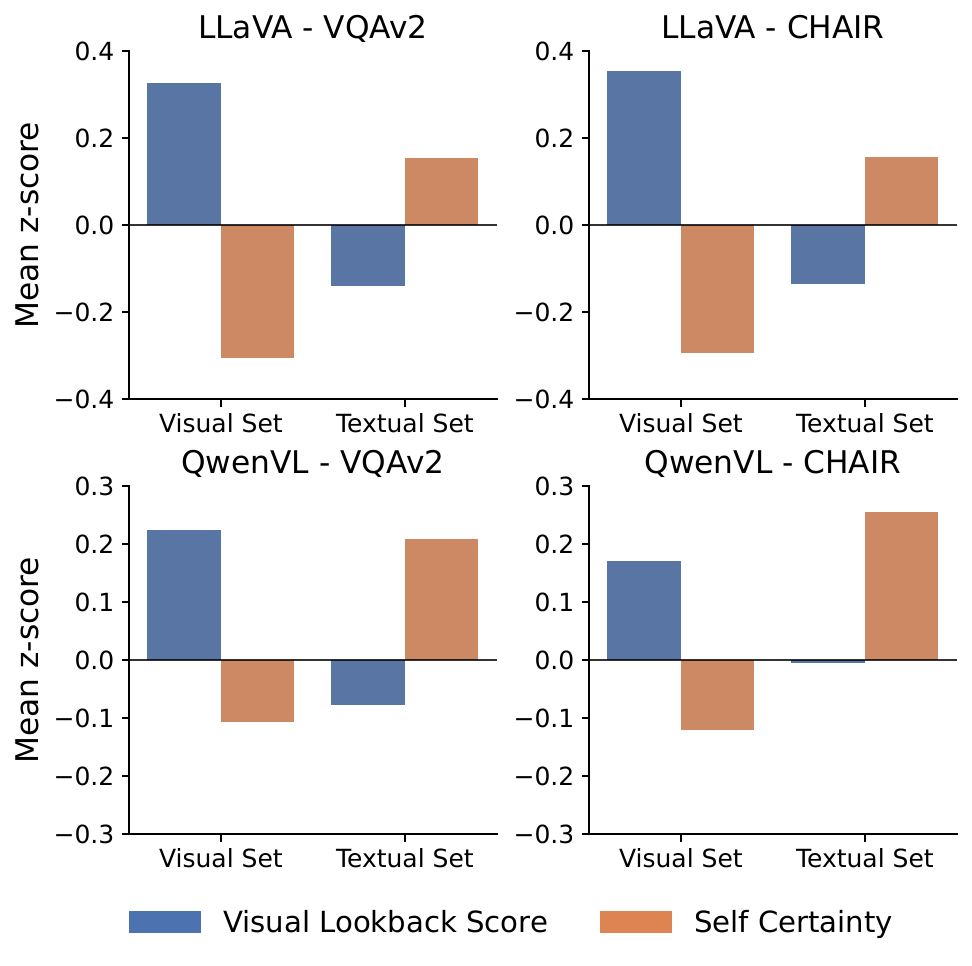}
    \caption{\textbf{Token-level POS analysis of visual lookback score and Self-Certainty.}
    We group generated words into a visual set and a textual set, and report the mean z-score of each signal within each group.
    Visual lookback scores are higher for the visual set, whereas Self-Certainty is higher for the textual set.
    }
    \label{fig:motivation_pos}
\end{figure}

%% file: Tex/Method.tex
\section{Method}
\label{sec:method}

Based on the observations in Section~\ref{sec:motivation}, our goal is to define a response score that converts output-space confidence into vision-aware confidence using only model-internal signals, without extra inference passes.
Specifically, we design the score according to following objectives:
\begin{enumerate}
    \item[O1.] \textbf{Preserve language confidence.}
    The score should retain token-level model confidence, so that implausible or low-probability responses are not favored.
    \item[O2.] \textbf{Incorporate visual reference usage.}
    Tokens that are likely under the output distribution should be further assessed by how strongly they interact with the visual reference.
    \item[O3.] \textbf{Aggregate according to visual relevance.}
    Since not all output tokens are equally informative for visual grounding, visually relevant token positions should contribute more.
\end{enumerate}

\paragraph{Preliminary.}

Let $p_\theta$ denote an LVLM parameterized by $\theta$. Given a text query $\mathbf{x}$ and vision tokens $\mathbf{v}$ encoded from an input image $\mathcal{I}$, the model generates a response $\mathbf{y}=(y_1,\ldots,y_{T})$ autoregressively:
\begin{equation}
\label{eq:preliminary_generation}
    p_\theta(\mathbf{y}\mid \mathbf{x},\mathbf{v})
    =
    \prod_{t=1}^{T}
    p_\theta(y_t \mid \mathbf{x},\mathbf{v},y_{<t}),
\end{equation}
and we denote the per-token probability as $p_t = p_\theta(y_t \mid \mathbf{x},\mathbf{v},y_{<t})$.
Under stochastic decoding, we sample $N$ candidate responses $\{\mathbf{y}^{(i)}\}_{i=1}^{N}$ for the same input $(\mathbf{x},\mathbf{v})$, and select the best candidate under a training-free score $S$:
\begin{equation}
\label{eq:preliminary_bon}
    \mathbf{y}^{*} = \arg\max_{i} \, S(\mathbf{y}^{(i)}\mid\mathbf{x},\mathbf{v}).
\end{equation}

\subsection{Visual Lookback Score}
\label{sec:method_1}

For each generated token \(y_t\), we estimate its \emph{visual lookback score}: how strongly the prediction step attends back to the original vision tokens.
Let $C_t$ denote the full causal context at the prediction step for $y_t$,
comprising text query tokens, vision tokens, and previous output tokens.
Let $\mathcal{P}_v \subset C_t$ denote the positions corresponding to vision tokens~$\mathbf{v}$.
Let $a_{t,k}^{(\ell,h)}$ be the attention weight from the query position predicting $y_t$
to context position $k$ at layer $\ell$ and head $h$.
We define the token-level \emph{visual lookback score} $A_t$ as:
\begin{equation}
\label{eq:visual_lookback_score}
    A_t
    =
    \frac{1}{LH}
    \sum_{\ell=1}^{L}
    \sum_{h=1}^{H}
    \frac{
    \sum_{k \in \mathcal{P}_v} a_{t,k}^{(\ell,h)}
    }{
    \sum_{k \in C_t} a_{t,k}^{(\ell,h)}
    },
\end{equation}
where $L$ and $H$ denote the number of layers and attention heads, respectively.
Since $C_t$ is the full causal context, the denominator equals one under softmax normalization; we retain the ratio form to make explicit that $A_t$ measures the fraction of attention directed to vision tokens. 
As $A_t$ is obtained from the attention weights produced during the generation forward pass, it does not require additional inference or external model.
A larger $A_t$ indicates greater \emph{visual reference usage}, meaning that the prediction step allocates a larger fraction of its attention to the explicit vision tokens.

\subsection{Lookback-Calibrated Token Score}
\label{sec:method_2}

Since $p_t$ and $A_t$ capture complementary aspects of token generation, we define a token-level score that preserves output-space confidence while accounting for visual lookback score. 
We define the lookback-calibrated token score $u_t$ as
\begin{equation}
\label{eq:lookback-calibrated_token_score}
    u_t = \log (p_t) + \alpha \log (A_t),
\end{equation}
equivalently written as $u_t = \log(p_t \cdot A_t^\alpha)$. 
The hyperparameter $\alpha$ controls the strength of this lookback calibration.
When $\alpha=0$, $u_t$ reduces to pure output-space confidence $\log p_t$~(O1).
When $\alpha>0$, $u_t$ is larger for tokens that are both highly probable and strongly attend to vision tokens, and smaller for tokens whose confidence is weakly supported by the visual reference~(O1, O2).

\subsection{Weighting Tokens with Visual Relevance}
\label{sec:method_3}
To obtain a response-level score, we aggregate token-level scores $\{u_t\}_{t=1}^T$ under a weight distribution $q$ over output token positions:
\begin{equation}
\label{eq:general_response_score}
    S(\mathbf{y}\mid\mathbf{x},\mathbf{v}) := \mathbb{E}_{t \sim q} \big[u_t\big] = \sum_{t=1}^{T} q(t) \, u_t.
\end{equation}
Uniform averaging corresponds to $q(t)=\tfrac{1}{T}$, but it assumes every token is equally informative, which is often not the case. 
In visually grounded responses, tokens such as object names, attributes, counts, and relations are generally more diagnostic than function words or generic phrases. 
The aggregation distribution should therefore concentrate on positions that interact more strongly with the visual reference, motivating $q$ derived from $A_t$.

We define \(q_\lambda\) as a visual relevance distribution over output tokens, derived as the solution to an entropy-regularized relevance maximization:
\begin{equation}
\label{eq:entropy_regularized_relevance}
    q_\lambda
    =
    \arg\max_{q \in \Delta_T}
    \Big[
    \lambda\,
    \mathbb{E}_{t \sim q}\big[\log (A_t)\big]
    +
    H(q)
    \Big],
\end{equation}
where $\Delta_T$ is the probability simplex over $T$ output positions, and $H(q)=-\sum_t q(t)\log \big(q(t)\big)$ is the entropy of $q$.
The first term encourages $q$ to concentrate on positions with high visual lookback score, while the entropy term keeps the distribution smooth and prevents hard selection. 
The hyperparameter $\lambda$ controls how sharply $q$ concentrates on high-$A_t$ positions.

\begin{proposition}
\label{prop:q_lambda}
    The solution to the entropy-regularized relevance maximization is
    \begin{equation}
    \label{eq:qlambda_closed_form}
        q_\lambda(t)
        =
        \frac{A_t^\lambda}{\sum_{j=1}^{T} A_j^\lambda}.
    \end{equation}
\end{proposition}
\begin{proof}
    See Appendix~\ref{appendix:theory_proof_q}.
\end{proof}
When $\lambda=0$, $q_\lambda$ is uniform; as $\lambda$ increases, $q_\lambda$ assigns more mass to positions with larger visual lookback score~(O3).

\noindent\textbf{Final score.} Instantiating $S$ with $q_\lambda$ gives the final response score \name{}:
\begin{equation}
\label{eq:final_score}
\begin{aligned}
    S(\mathbf{y}\mid\mathbf{x},\mathbf{v})
    &=
    \mathbb{E}_{t \sim q_\lambda}
    \left[
    \log \big(p_t\big)
    +
    \alpha \log \big(A_t\big)
    \right] \\
    &=
    \sum_{t=1}^{T}
    \frac{
    A_t^\lambda
    \left(
    \log \big(p_t\big)
    +
    \alpha \log \big(A_t\big)
    \right)
    }{
    \sum_{j=1}^{T}
    A_j^\lambda
    }.
\end{aligned}
\end{equation}
\name{} measures the expected lookback calibrated token score under a visual relevance distribution. 
Intuitively, it evaluates model confidence in a response while emphasizing tokens that interact more strongly with the explicit visual reference.

%% file: Tex/Experiments.tex
\section{Experiments}
\subsection{Setups}

\paragraph{Benchmarks and metrics.}
We evaluate on four image-grounded benchmarks suited to Best-of-$N$ response selection: VQAv2~\citep{goyal2017vqav2}, CHAIR~\citep{rohrbach2018chair}, AMBER~\citep{wang2023amber}, and HallusionBench~\citep{guan2024hallusionbench}.
We chose this suite so that it (1) requires visual evidence beyond language plausibility, (2) provides sufficient candidate diversity and selection headroom, and (3) covers complementary response formats, including short-answer QA, open-ended generation, and discriminative visual-grounding evaluation.
For a consistent main comparison, we report a single higher-is-better primary metric per benchmark: accuracy for VQAv2, F1 for CHAIR and AMBER, and GPT-evaluated correctness for HallusionBench.
See Appendix~\ref{appendix:benchmark_details} for details.

\input{Tables/main_results}

\paragraph{Models and baselines.}
For the experiments, we consider three representative LVLMs: LLaVA-1.5-7B~\citep{liu2023visual}, Qwen2.5-VL-7B~\citep{bai2025qwen2}, and InternVL3-8B~\citep{zhu2025internvl3}. 
Also, to evaluate the effectiveness of \name{}, we consider two categories of baselines: language side and vision side.
(1) \textit{Language side}: Self-Certainty (SC)~\citep{kang2025scalable} measures model's token distributional confidence leveraging KL divergence with uniform distribution.
Universal Self-Consistency (USC)~\citep{chen2023universal} select most consistent response by prompting all candidates responses to model.
(2) \textit{Vision side}:
CLIPScore~\citep{hessel2021clipscore} leverages a pretrained vision-language encoder to measure the cosine similarity between the text embedding of a generated response and the corresponding image embedding.
VAUQ~\citep{park2026vauq} estimates model uncertainty by masking a fixed proportion of visual attention weights and computing the entropy of the resulting output distribution.
To evaluate the gain of each method, we additionally consider the simplest baseline, random selection (Random). 

\paragraph{Implementation details.} 
For all models, to ensure sufficient diversity among candidates, we generate $N=25$ candidate responses per input using nucleus sampling \citep{holtzman2019curious} with temperature=1.2 and top-p=0.9. 
\name{} has two hyperparameters, $\alpha$ and $\lambda$, which control the strength of token-level visual lookback calibration and response-level visual relevance weighting, respectively. We set these hyperparameters on a per-model basis. Specifically, we use $(\alpha, \lambda)=(7.0, 1.5)$ for LLaVA-1.5-7B, $(0.5, 1.25)$ for Qwen2.5-VL-7B, and $(0.25, 1.25)$ for InternVL3-8B.
{
See Appendix~\ref{appendix:experimental_details} for benchmark, baseline, and hyperparameter selection details.
}

\subsection{Main Results}
Table~\ref{tab:main_results} reports Best-of-$N$ selection performance for $N=5$ and $N=25$ across four benchmarks and three LVLMs.
Across all setups, \name{} achieves the highest model-wise average for every LVLM, improving over the strongest competing baseline by $+1.31$ on LLaVA-1.5-7B (over SC), $+1.74$ on Qwen2.5-VL-7B (over VAUQ), and $+0.69$ on InternVL3-8B (over USC).

While SC achieves competitive performance as an output-space baseline, \name{} improves over SC in most settings and achieves higher average scores.
This suggests that visually grounded response selection benefits from measuring not only how confident a response is, but whether that confidence is supported by visual reference usage.

USC is also competitive in InternVL3-8B, but its gains are less stable across models and candidate budgets.
Since USC relies on the LVLM itself to select candidate responses, its effectiveness may depend on the model's own selection ability, not just on the quality of the sampled candidates.

Compared with VAUQ which estimates reliability from a perturbed visual condition, \name{} uses visual evidence in a more candidate-specific way;
\name{} directly measures whether high-confidence tokens look back to the visual input..
This makes \name{} well suited to Best-of-$N$ selection, where the score must distinguish candidates whose confident claims are visually supported from those that are merely plausible.

\input{Figures/n_scaling}

\subsection{Scaling Results}
We further investigate how each scoring method benefits from varied candidate budget. 
Figure~\ref{fig:N_Scaling} shows Best-of-$N$ accuracy as $N$ increases from 1 to 25 on HallusionBench for LLaVA-1.5-7B and Qwen2.5-VL-7B. 
\name{} maintains a consistent advantage over all baselines across the full range of $N$. 
While SC improves gradually with $N$, \name{} maintains a clear margin as the candidate pool grows, especially on Qwen2.5-VL-7B.
CLIPScore and VAUQ do not consistently benefit from larger candidate pools.
This suggests that scaling Best-of-$N$ requires a reliable grounding-aware score, rather than simply increasing the number of candidates.
Overall, \name{} shows stable performance gains as $N$ increases.

\input{Figures/ablation}

\subsection{Additional Analyses}
\paragraph{Ablation study.}

Figure~\ref{fig:ablation} ablates the two hyperparameters of \name{} on AMBER with LLaVA-1.5-7B: $\alpha$ for token-level visual lookback calibration and $\lambda$ for response-level visual relevance weighting.
In the top panel, increasing $\alpha$ yields modest gains when $\lambda=0$ because aggregation remains uniform.
With $\lambda=1.5$, however, the same increase leads to stronger performance, showing that token-level calibration is more effective when visually relevant tokens are weighted more heavily.
In the bottom panel, increasing $\lambda$ improves performance when $\alpha=0$, but with $\alpha=7$, performance peaks at a moderate $\lambda$ and declines as the aggregation becomes too sharp.
These results show that token-level lookback calibration and visual relevance-based aggregation are complementary.
\paragraph{Random image control.}
Beyond removing the image in Section~\ref{sec:motivation_1}, we also replace it
with a randomly sampled one, so that the scoring context remains a well-formed
image--text input and only the visual evidence is mismatched. With the candidate
set frozen and LLaVA-1.5-7B at $N=25$, this replacement removes 90.0--112.5\% of
\name{}'s gain over random selection, whereas SC retains roughly half
of its gain on VQAv2 and is essentially unchanged on CHAIR.
Appendix~\ref{appendix:random_image} reports the protocol together with top-1
agreement between original-image and random-image scoring.

\paragraph{Robustness to hyperparameters.}
The main results use per-model $(\alpha,\lambda)$, so we check whether the
conclusions depend on this choice. Under a single global setting
$(\alpha,\lambda)=(0.25,1.25)$ shared by every model and benchmark,
\name{} still obtains the highest model-wise average on all three
LVLMs and loses only 0.19 points relative to the reported configuration. The
per-model values mainly absorb the different numerical scales of $\log p_t$ and
$\log A_t$ across architectures rather than fitting each setting.
Appendix~\ref{appendix:global_hp} reports the global-setting results, and
Appendix~\ref{appendix:sensitivity} sweeps both hyperparameters over the full search space.

\paragraph{Variance across random seeds.}
Since candidates are sampled stochastically, we repeat the $N=5$ setting on VQAv2 and CHAIR with two additional seeds.
\name{} obtains the highest mean in all four model--benchmark pairs, while the strongest competing baseline changes across settings. Per-seed means and standard deviations for all methods are reported in Appendix~\ref{appendix:seed}.

\paragraph{Stronger evaluators and human preference.}
A pairwise LVLM-as-a-judge selector~\citep{chen2024mllmjudge} improves over the batch-style USC selector, but remains below \name{} on every benchmark; Appendix~\ref{appendix:judge} reports the tournament comparison results.
On the CHAIR instances where \name{} and SC select different responses, human annotators prefer the \name{}-selected response in $50.3\%$ of judgments versus $35.0\%$ for SC, while automatic CHAIR-F1 ties nearly half of these cases; Appendix~\ref{appendix:human} reports the annotation setup and the per-sample breakdown.

\input{Figures/computation}
\paragraph{Scoring overhead.}

To evaluate the computational efficiency of \name{}, we measure post-generation scoring overhead in milliseconds per response on CHAIR across three LVLMs.
As shown in Figure~\ref{fig:computation}, \name{} adds only modest overhead and remains substantially cheaper than VAUQ and USC, which require perturbation-based uncertainty estimation or additional candidate-comparison passes.
This shows that \name{} provides a favorable efficiency--grounding trade-off: unlike SC, it incorporates an explicit visual-grounding signal, yet unlike heavier visual scorers, it obtains this signal directly from the LVLM's internal likelihood and attention statistics without external models or expensive post-hoc verification.
See Appendix~\ref{appendix:computation} for measurement setup and hardware configuration.

%% file: Tables/main_results.tex
\begin{table*}[t]
\begin{center}
\begin{small}
\setlength{\tabcolsep}{5pt}
\caption{\textbf{Main results.} 
Best-of-$N$ selection performance for $N=5$ and $N=25$ across four benchmarks and three LVLMs. Average denotes the average score over all benchmark--$N$ combinations within each model. Higher scores ($\uparrow$) indicate better performance. The best results in each setup are \textbf{bolded}, and the second-best are \underline{underlined}.}
\label{tab:main_results}
\resizebox{0.98\textwidth}{!}{
\begin{tabular}{llcccccccc|c}
    \toprule
    \multirow{2}{*}[-0.5ex]{\textbf{Model}}
    & \multirow{2}{*}[-0.5ex]{\textbf{Method}}
    & \multicolumn{2}{c}{\textbf{VQAv2}}
    & \multicolumn{2}{c}{\textbf{CHAIR}}
    & \multicolumn{2}{c}{\textbf{HallusionBench}}
    & \multicolumn{2}{c|}{\textbf{AMBER}}
    & \multirow{2}{*}[-0.5ex]{\textbf{Average}} \\
    \cmidrule(lr){3-4}\cmidrule(lr){5-6}\cmidrule(lr){7-8}\cmidrule(lr){9-10}
    & & $N=5$ & $N=25$ & $N=5$ & $N=25$ & $N=5$ & $N=25$ & $N=5$ & $N=25$ & \\
    \midrule

    \multirow{6}{*}{\textbf{LLaVA-1.5-7B}}
    & Random     & \tc{62.57} & \tc{60.40} & \tc{70.20} & \tc{71.52} & \tc{44.84} & \tc{43.82} & \tc{63.90} & \tc{64.08} & \tc{60.16} \\
    & SC         & \tc{65.93} & \tc{65.70} & \tb{72.76} & \tb{73.53} & \tb{47.21} & \tb{46.19} & \tc{64.10} & \tc{63.45} & \tb{62.36} \\
    & USC        & \tc{63.78} & \tc{61.62} & \tc{71.07} & \tc{70.87} & \tc{43.09} & \tc{43.20} & \tc{64.61} & \tc{63.96} & \tc{60.28} \\
    & CLIPScore  & \tc{64.07} & \tc{62.27} & \tc{71.89} & \tc{71.93} & \tc{43.15} & \tc{43.65} & \ta{66.08} & \tb{65.25} & \tc{61.03} \\
    & VAUQ       & \tb{66.03} & \tb{66.53} & \tc{72.48} & \tc{73.15} & \tc{45.01} & \tc{45.35} & \tc{64.28} & \tc{64.22} & \tc{62.13} \\
    & \ca{LookBack (Ours)}  & \ca{66.63} & \ca{67.60} & \ca{74.03} & \ca{74.43} & \ca{47.38} & \ca{48.22} & \cb{65.33} & \ca{65.74} & \ca{63.67} \\
    \midrule
    
    \multirow{6}{*}{\textbf{Qwen2.5-VL-7B}}
    & Random     & \tc{63.40} & \tc{62.23} & \tc{74.18} & \tc{73.22} & \tc{57.19} & \tc{57.36} & \tc{74.19} & \tc{73.84} & \tc{66.95} \\
    & SC         & \tc{66.33} & \tc{65.47} & \tc{74.41} & \tc{71.97} & \tb{57.36} & \tc{58.04} & \tc{74.43} & \tc{74.30} & \tc{67.78} \\
    & USC        & \tb{66.88} & \tc{66.73} & \tc{74.63} & \tc{73.96} & \tc{56.80} & \tc{56.91} & \tc{74.33} & \tb{74.64} & \tc{68.11} \\
    & CLIPScore  & \tc{64.07} & \tc{62.17} & \tc{74.47} & \tb{74.80} & \tb{57.36} & \tb{58.21} & \tc{74.17} & \tc{74.35} & \tc{67.45} \\
    & VAUQ       & \tc{66.67} & \tb{66.83} & \tb{75.15} & \tc{73.24} & \tc{57.02} & \tc{57.19} & \tb{74.77} & \tc{74.50} & \tb{68.17} \\
    & \ca{LookBack (Ours)}  & \ca{68.23} & \ca{67.47} & \ca{75.42} & \ca{75.28} & \ca{60.07} & \ca{61.93} & \ca{74.98} & \ca{75.92} & \ca{69.91} \\
    \midrule

    \multirow{6}{*}{\textbf{InternVL3-8B}}
    & Random     & \tc{62.17} & \tc{61.33} & \tc{77.88} & \tc{77.78} & \tc{56.35} & \tc{57.19} & \tc{79.82} & \tc{79.45} & \tc{69.00} \\
    & SC         & \tc{66.70} & \tc{66.67} & \tc{78.35} & \tc{78.24} & \tc{56.35} & \tc{56.85} & \tb{81.66} & \tb{81.67} & \tc{70.81} \\
    & USC        & \ta{70.33} & \tc{69.39} & \tc{78.34} & \ta{78.28} & \tb{58.15} & \tb{58.38} & \tc{80.04} & \tc{79.72} & \tb{71.58} \\
    & CLIPScore  & \tc{63.47} & \tc{63.47} & \tc{78.32} & \tb{78.27} & \ta{59.39} & \ta{59.39} & \tc{79.68} & \tc{79.98} & \tc{70.25} \\
    & VAUQ       & \tc{68.23} & \tb{69.77} & \tb{78.49} & \ta{78.28} & \tc{57.19} & \tc{57.19} & \tc{81.47} & \tc{81.65} & \tc{71.53} \\
    & \ca{LookBack (Ours)}  & \ca{70.33} & \ca{72.57} & \ca{79.37} & \cc{78.22} & \cc{57.02} & \cc{56.85} & \ca{81.77} & \ca{82.07} & \ca{72.27} \\
    \bottomrule
    
\end{tabular}
}
\end{small}
\end{center}
\end{table*}

%% file: Figures/n_scaling.tex
\begin{figure}[t]
    \includegraphics[width=\linewidth]{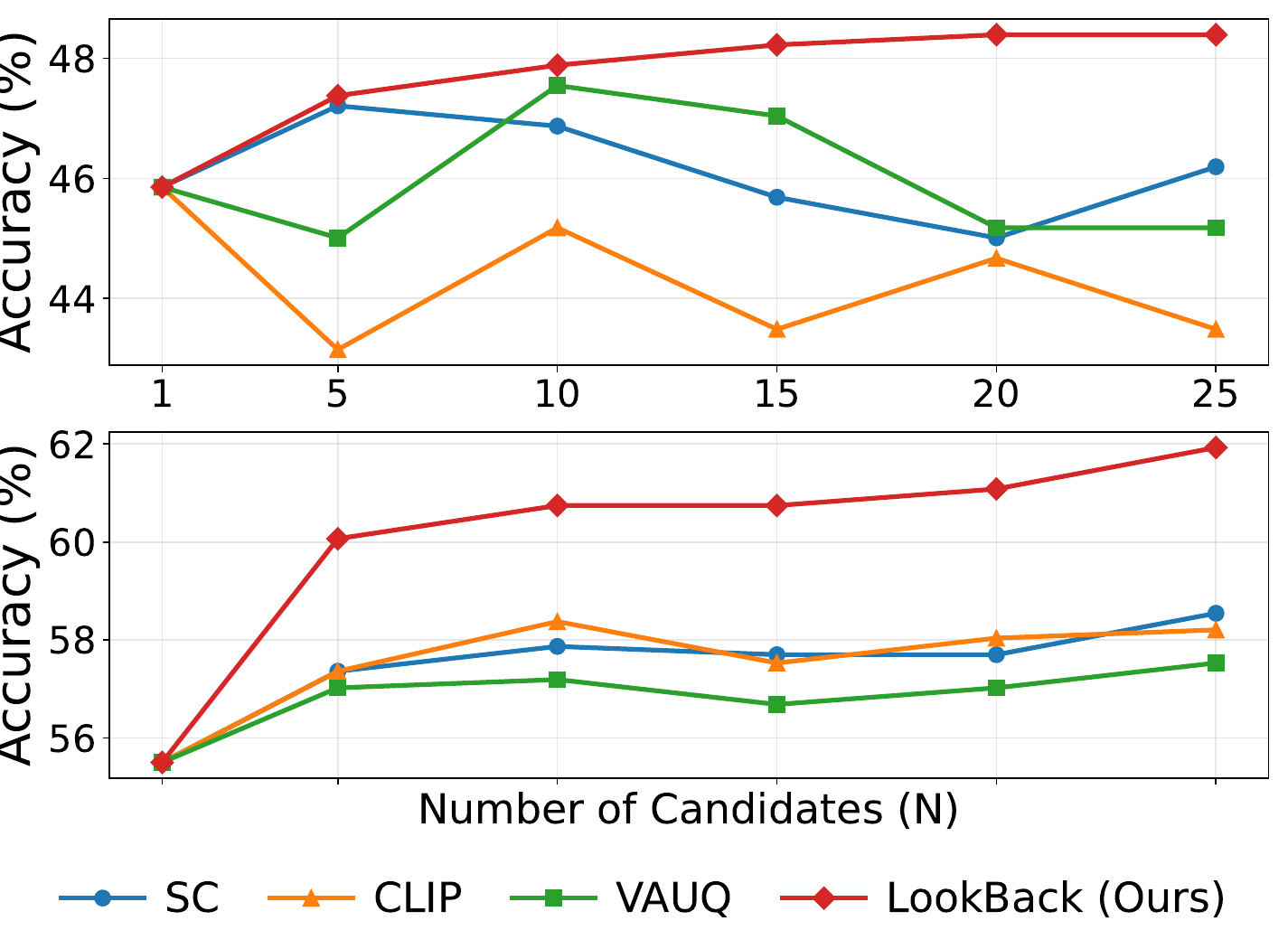}
    \caption{\textbf{Best-of-$N$ scaling.}
    We report Best-of-$N$ selection performance as the candidate pool size increases from $N{=}1$ to $N{=}25$ on HallusionBench with LLaVA-1.5-7B (top) and Qwen2.5-VL-7B (bottom).}
    \label{fig:N_Scaling}
\end{figure}

%% file: Figures/ablation.tex
\begin{figure}[t]
    \includegraphics[width=\linewidth]{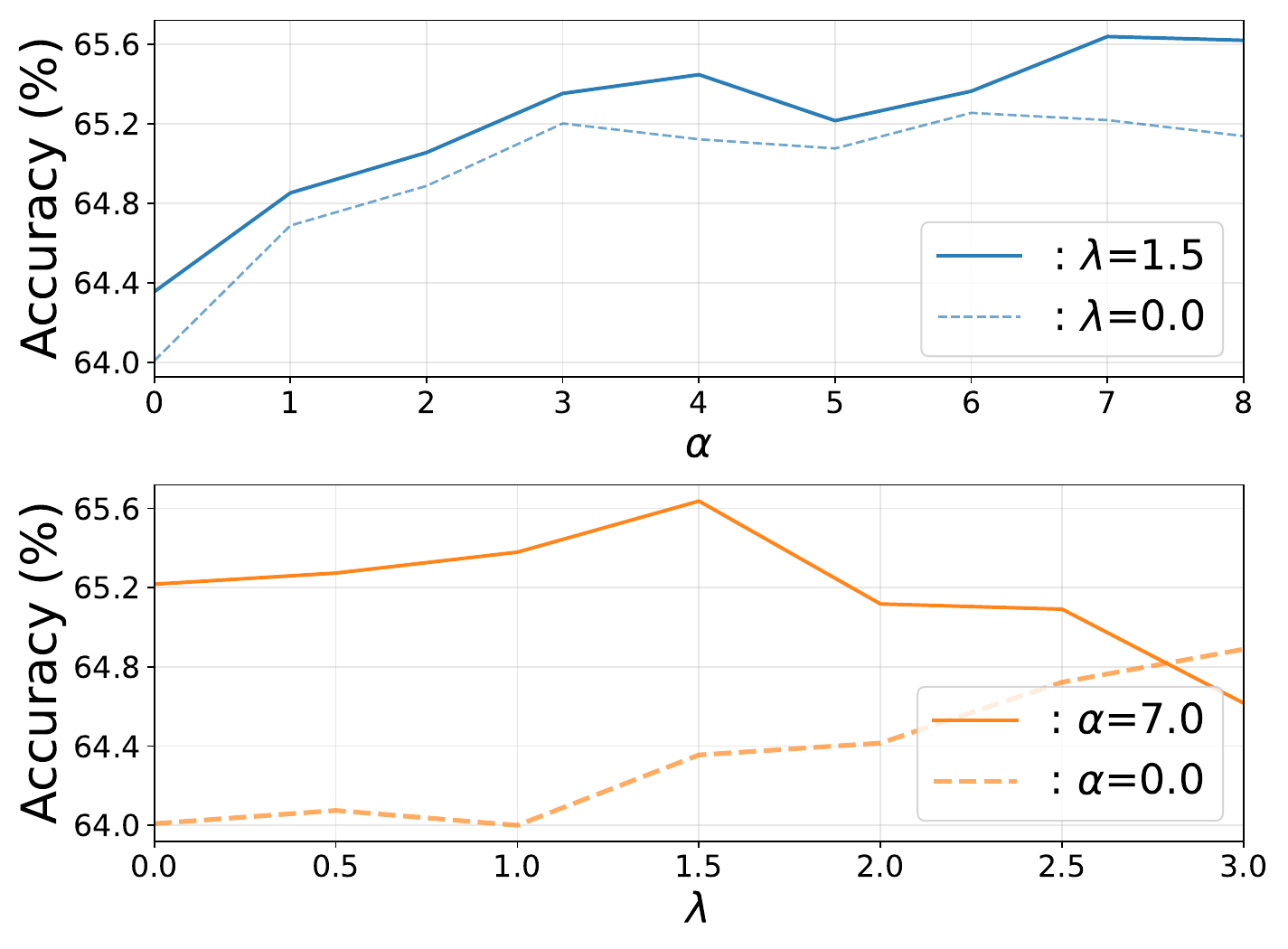}
    \caption{\textbf{Ablation study.}
        We ablate token-level calibration strength $\alpha$ and the response-level visual weighting strength $\lambda$ on AMBER with LLaVA-1.5-7B.
        Top: performance when varying $\alpha$ with fixed $\lambda$.
        Bottom: performance when varying $\lambda$ with fixed $\alpha$.
        Dashed lines show the corresponding variants without the other component.
    }
    \label{fig:ablation}
\end{figure}

%% file: Figures/computation.tex
\begin{figure}[t]
    \includegraphics[width=\linewidth]{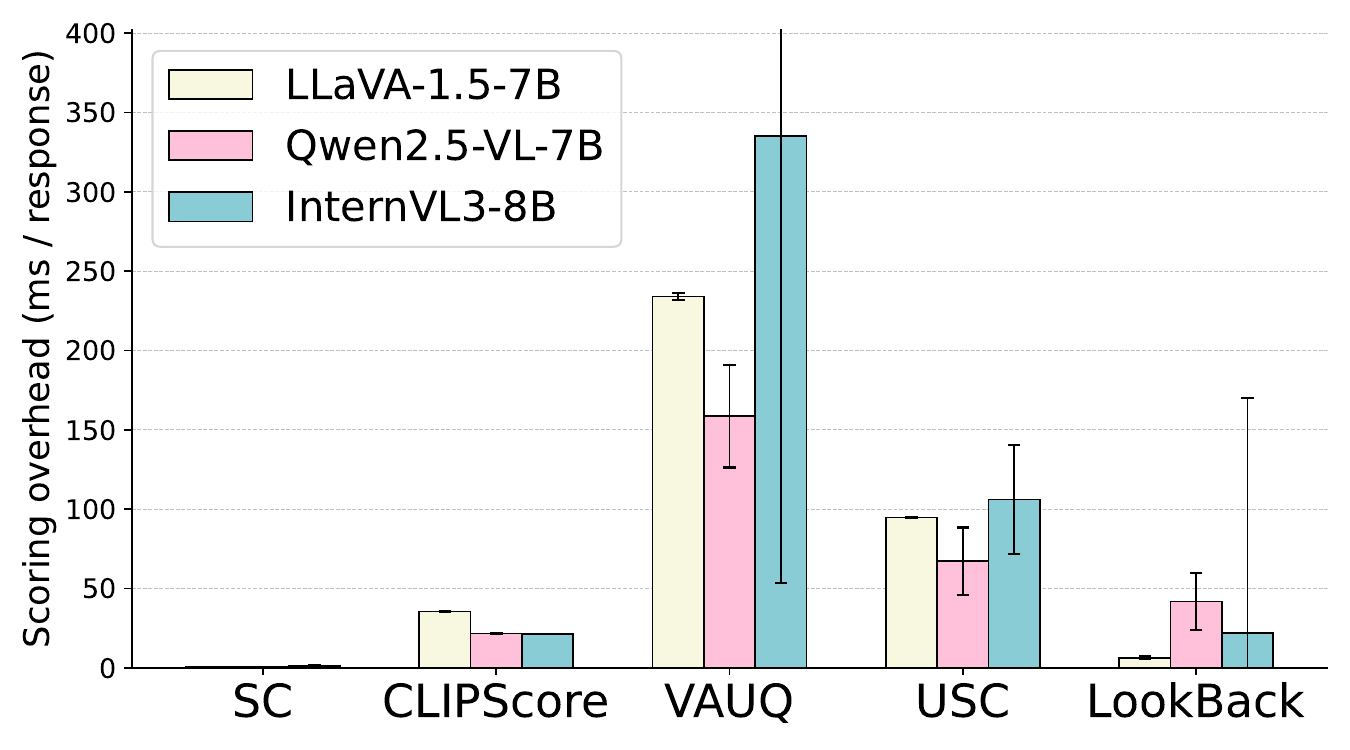}
    \caption{\textbf{Scoring overhead.}
    We compare the average scoring time per response across three LVLMs on CHAIR.
    Error bars indicate standard deviation.
    }
    \label{fig:computation}
\end{figure}

%% file: Tex/Conclusion.tex
\section{Conclusion}

We presented \name{}, a training-free response scoring method for LVLMs that calibrates output-space confidence with visual lookback. 
Our motivation analysis shows that confidence-based scoring can remain largely insensitive to the image, even when computed under image-conditioned generation, and therefore tends to capture textual plausibility rather than visual grounding. 
To address this gap, \name{} combines token likelihood with a token-level visual lookback score and aggregates these calibrated token scores according to a visual relevance distribution. 
Across multiple LVLMs and visual understanding benchmarks, \name{} consistently improves Best-of-$N$ selection over both linguistic and vision side baselines, while requiring no auxiliary verifier, training, or extra inference passes.

\section*{Limitations}

\name{} requires access to internal attention weights, which limits its applicability to black-box LVLMs.
Beyond this access requirement, visual lookback is only a proxy for visual reference usage, not a guarantee of factual correctness: a model can strongly attend to image tokens while still generating an incorrect or unsupported response.
The reliability of this proxy may also depend on model-specific attention behavior, as attention distributions and their calibration can vary across LVLM architectures.
Consistent with this, the residual errors presented in Appendix~\ref{appendix:failure} concentrate on relational and comparative judgments, where the response refers to the relevant entities but misjudges how they compare.

Our evaluation also covers a specific regime: image-grounded Best-of-$N$ selection with relatively concise responses. Whether \name{} generalizes to long-form multimodal reasoning, multi-image or video inputs, and more complex grounding scenarios remains an important direction for future work.

Two extensions are naturally expected from the formulation. 
First, \name{} leaves generation unchanged and ranks whatever candidate pool a decoder produces, so moving the score to generation time is a natural next step: since $q_\lambda$ is defined over a completed response, an inference-time variant would need to estimate online which generation steps require visual grounding. 
Second, the same principle may apply whenever a response should be grounded in a designated part of the input context, such as retrieved documents in RAG, tool outputs, or instruction tokens for prompt following. 
In such settings, a reliable scorer should assess not only whether a response is likely under the model's output distribution, but also whether high-confidence generation steps refer back to the intended source. 
We believe this perspective opens a path toward lightweight, model-internal scoring methods for source-grounded generation beyond the visual domain.

\section*{Broader Impact and Ethical Implications}

\name{} is intended to improve the reliability of LVLM response selection by favoring candidates whose high-confidence tokens are more strongly tied to the visual input.
This direction may be beneficial in applications where visually unsupported responses can mislead users, such as image-based assistance, educational tools, accessibility interfaces, and content analysis systems.
Because \name{} requires neither additional training nor an auxiliary evaluator, it may also provide a lightweight mechanism for improving response selection in resource-constrained settings.

At the same time, \name{} should not be interpreted as a guarantee of correctness or safety.
Visual lookback is only a proxy for reference usage: a response can strongly attend to image tokens while still being incorrect, biased, or harmful.
These risks are especially consequential in high-stakes applications, including medical, legal, surveillance, and accessibility-critical settings.
We therefore recommend using \name{} only as a response-selection aid, together with task-specific validation and human oversight.

\section*{Acknowledgments}

Dongseok Lee and Jaehyung Kim are affiliated with the Department of Artificial Intelligence at Yonsei University.
This research was supported in part by Institute for Information \& communications Technology Planning \& Evaluation (IITP) grant funded by the Korea government (MSIT) (No. RS-2020-II201361, Artificial Intelligence Graduate School Program (Yonsei University); No. RS-2025-25442405, Development of a Self-Learning World Model-Based AGI System for Hyperspectral Imaging).

%% file: Tex/Appendix.tex
\clearpage
\appendix

\section{Additional Theoretical Analysis}
\label{appendix:theory}

\subsection{Log-Linear Product-of-Experts Interpretation}
\label{appendix:theory_poe}

The grounded token score can be written as
\begin{equation*}
    u_t
    =
    \log (p_t)
    +
    \alpha \log (A_t)
    =
    \log \big(p_t \cdot A_t^\alpha\big).
\end{equation*}
This resembles a product-of-experts-style score~\cite{hinton2002training},
where output-space confidence \(p_t\) and visual lookback score \(A_t\)
serve as complementary factors that are combined multiplicatively into an
unnormalized token-level score.
The hyperparameter \(\alpha\) controls the relative strength of the visual
factor.
Thus, tokens with low visual lookback score receive a lower combined score
even when their output probability is high.

\subsection{Proof of Proposition~\ref{prop:q_lambda}}
\label{appendix:theory_proof_q}

We prove that \(q_\lambda\) defined by
\begin{equation*}
    q_\lambda
    =
    \arg\max_{q \in \Delta_T}
    \Big[
    \lambda\,\mathbb{E}_{\tau \sim q}\big[\log (A_\tau)\big]
    +
    H(q)
    \Big]
\end{equation*}
has the closed-form solution
\begin{equation*}
    q_\lambda(t)
    =
    \frac{A_t^\lambda}{\sum_{j=1}^{T} A_j^\lambda}.
\end{equation*}
Here,
\begin{equation*}
\Delta_T
=
\left\{
q\in\mathbb{R}_{\ge 0}^{T}
:
\sum_{t=1}^{T}q(t)=1
\right\}
\end{equation*}
is the set of distributions over output positions.

\begin{proof}
Expanding the objective using the definition of \(H(q)\), the problem becomes
\begin{equation*}
    \max_{q \in \Delta_T}
    \sum_{t=1}^{T} q(t)
    \Big[
    \lambda \log (A_t) - \log (q(t))
    \Big].
\end{equation*}
Assuming \(A_t>0\) for all \(t\), this is a strictly concave maximization
over the probability simplex \(\Delta_T\), so a unique solution exists.
Introducing a Lagrange multiplier \(\eta\) for the constraint
\(\sum_t q(t)=1\), the Lagrangian is
\begin{equation*}
\begin{aligned}
    \mathcal{L}
    =
    \sum_{t=1}^{T} q(t)\lambda \log (A_t)
    & -
    \sum_{t=1}^{T} q(t)\log (q(t)) \\
    & +
    \eta\!\left(\sum_{t=1}^{T} q(t)-1\right).
\end{aligned}
\end{equation*}
Setting \(\partial\mathcal{L}/\partial q(t)=0\) gives
\begin{equation*}
    \lambda \log (A_t) - \log (q(t)) - 1 + \eta = 0,
\end{equation*}
which implies
\begin{equation*}
    \log (q(t))
    =
    \lambda \log (A_t)
    +
    (\eta - 1).
\end{equation*}
Therefore,
\begin{equation*}
    q(t)
    \propto
    A_t^\lambda.
\end{equation*}
Applying the normalization constraint \(\sum_t q(t)=1\) yields
\begin{equation*}
    q_\lambda(t)
    =
    \frac{A_t^\lambda}{\displaystyle\sum_{j=1}^{T} A_j^\lambda}.
\end{equation*}
\end{proof}

\noindent\textbf{Interpretation.}
The exponent \(\lambda\) controls the sharpness of \(q_\lambda\).
When \(\lambda=0\), \(q_\lambda\) is the uniform distribution over output
positions, recovering uniform averaging.
As \(\lambda \to \infty\), \(q_\lambda\) concentrates on the position with
the largest \(A_t\), approaching hard selection of the most visually attended
token.
For finite \(\lambda>0\), \(q_\lambda\) acts as a soft selector that prefers
visually attended positions while maintaining nonzero weight on all positions.

\subsection{Geometric Mean Interpretation}
\label{appendix:theory_geomean}

We show that the proposed score \(S(\mathbf{y}\mid\mathbf{x},\mathbf{v})\)
can be interpreted as the log of a visual-relevance-weighted geometric mean.
This provides a justification for using a normalized weighted average rather
than an unnormalized sum over output tokens.

\begin{corollary}
\label{cor:geomean}
Let \(q_\lambda\) be as defined in Proposition~\ref{prop:q_lambda}.
Then
\begin{equation*}
    \exp\!\Big(S(\mathbf{y}\mid\mathbf{x},\mathbf{v})\Big)
    =
    \prod_{t=1}^{T}
    \Big(p_t \cdot A_t^{\alpha}\Big)^{q_\lambda(t)}.
\end{equation*}
That is, \(\exp(S)\) is the visual-relevance-weighted geometric mean of the
token-level product scores \(p_t \cdot A_t^\alpha\).
\end{corollary}

\begin{proof}
By definition,
\begin{equation*}
    S(\mathbf{y}\mid\mathbf{x},\mathbf{v})
    =
    \sum_{t=1}^{T} q_\lambda(t)
    \Big[
    \log (p_t)
    +
    \alpha \log (A_t)
    \Big].
\end{equation*}
Since
\begin{equation*}
    \log (p_t) + \alpha \log (A_t)
    =
    \log \big(p_t \cdot A_t^\alpha\big),
\end{equation*}
we have
\begin{equation*}
    S(\mathbf{y}\mid\mathbf{x},\mathbf{v})
    =
    \sum_{t=1}^{T} q_\lambda(t)
    \log \big(p_t \cdot A_t^\alpha\big).
\end{equation*}
Exponentiating both sides gives
\begin{equation*}
    \exp\!\Big(S(\mathbf{y}\mid\mathbf{x},\mathbf{v})\Big)
    =
    \prod_{t=1}^{T}
    \Big(p_t \cdot A_t^{\alpha}\Big)^{q_\lambda(t)}.
\end{equation*}
\end{proof}

\noindent\textbf{Implication.}
A standard log-likelihood sum, \(\sum_t \log (p_t)\), scales with response
length \(T\), making scores difficult to compare across candidates of
different lengths.
In contrast, the proposed score is equivalent to maximizing the
visual-relevance-weighted geometric mean
\begin{equation*}
    \prod_{t=1}^{T}
    \big(p_t \cdot A_t^\alpha\big)^{q_\lambda(t)},
\end{equation*}
where the exponents satisfy \(\sum_t q_\lambda(t)=1\).
Thus, the score compares candidates using a normalized weighted average of
token-level product scores rather than an unnormalized sum.
Unlike uniform length normalization with \(q(t)=1/T\), the weighting
\(q_\lambda\) adapts the normalization to visual relevance.

In particular, since the exponents \(q_\lambda(t)\) sum to one regardless of \(T\),
appending additional visually descriptive tokens redistributes weight among positions
rather than accumulating score, so the response score is not inflated by verbosity.

\subsection{Computational Cost of the Visual Lookback Score}
\label{appendix:cost}

\name{} reuses the attention weights produced during the generation forward pass, so it requires
no auxiliary model and no forward pass dedicated to scoring.
The remaining cost is that of exposing and aggregating those weights, which we analyze here and
measure in Appendix~\ref{appendix:computation}.

A naive implementation that returns the full attention tensors for all $T$ generated tokens, across $L$ layers and $H$ attention heads, stores
\begin{equation*}
    O\big(T \cdot L \cdot H \cdot |C_t|\big)
\end{equation*}
values, where $|C_t|$ is the context length at step $t$ and therefore grows with the number of
vision tokens.
This term dominates for long responses and for high-resolution inputs, which is the regime in
which attention-based scoring is most often criticized as expensive.

Equation~\ref{eq:visual_lookback_score}, however, depends on the attention weights only through
the mass assigned to vision-token positions $\mathcal{P}_v$.
Reducing over $\mathcal{P}_v$ inside the generation loop yields a single scalar $A_t$ per generated
token, so the statistic required by \name{} can be accumulated online and the additional memory is
\begin{equation*}
    O(T),
\end{equation*}
independent of $L$, $H$, and the number of vision tokens.
The additional computation is one reduction over $\mathcal{P}_v$ per decoding step, which is
included in the latency reported in Figure~\ref{fig:computation}.

\clearpage

\section{Additional Experimental Details}
\label{appendix:experimental_details}

\subsection{Benchmark Details}
\label{appendix:benchmark_details}
We select benchmarks that are informative for Best-of-$N$ response selection.
In particular, the candidate pool should contain meaningful variation in response quality, so that a response scorer has non-trivial room to improve over random selection.
We therefore consider both candidate diversity and oracle--average gaps when constructing the evaluation suite.
Table~\ref{tab:oracle_by_benchmark} reports oracle Best-of-$N$ performance, where the oracle selects the highest-scoring response among the sampled candidates for each input.
The large gap between oracle and average performance indicates that the sampled candidate pools contain substantial selection headroom.

For the main comparison, we report a single higher-is-better metric for each benchmark.
This avoids mixing heterogeneous native metrics in the main table while still allowing a consistent comparison of response scorers.
We provide benchmark-specific details below.

\paragraph{VQAv2.}
VQAv2~\citep{goyal2017vqav2} is a short-answer visual question answering benchmark where each question is paired with multiple human annotations.
We use VQAv2 to evaluate whether a scorer can select answers that are grounded in the visual content of the image.
For each generated response, we parse the predicted answer and evaluate it against the ground-truth annotations.
The primary metric is accuracy:
\[
\resizebox{1.0\columnwidth}{!}{$
    \text{Acc}(\textit{ans})
    =
    \min\left\{
        \frac{\#\text{ human annotators that provided } \textit{ans}}{3},
        1
    \right\}.
$}
\]
Dataset-level accuracy is computed by averaging the per-question scores.

\paragraph{CHAIR.}
CHAIR~\citep{rohrbach2018chair} evaluates object hallucination in open-ended image descriptions.
Following the CHAIR setup, generated captions are parsed into object mentions and matched to MS-COCO object categories using the benchmark's synonym mapping.
For each image, let $G_i$ denote the set of ground-truth objects and $P_i$ denote the set of predicted object mentions extracted from the generated response.
A predicted object is counted as hallucinated if it does not appear in $G_i$.
Although CHAIR conventionally reports CHAIR$_S$ and CHAIR$_I$, we report F1 as the primary metric to obtain a single higher-is-better score for Best-of-$N$ selection:
\[
    \text{Prec}_i = \frac{|P_i \cap G_i|}{|P_i|},
    \qquad
    \text{Rec}_i = \frac{|P_i \cap G_i|}{|G_i|},
\]
\[
    \text{F1}_i =
    \frac{2 \cdot \text{Prec}_i \cdot \text{Rec}_i}
    {\text{Prec}_i + \text{Rec}_i},
    \qquad
    \text{F1} = \frac{1}{|\mathcal{D}|}\sum_{i \in \mathcal{D}} \text{F1}_i.
\]
When the denominator is zero, the corresponding precision, recall, or F1 term is set to zero.

\input{Tables/oracle_results}

\paragraph{AMBER.}
AMBER~\citep{wang2023amber} is a multidimensional hallucination benchmark for LVLMs.
In the generative setting, AMBER evaluates whether generated responses contain visual instances that are supported by the image.
Unlike CHAIR, which focuses on object hallucination, AMBER covers a broader set of visual error types, including object existence, attributes, and relations.
For each sample, we compare generated visual instances against the ground-truth visual instance set.
Instances absent from the ground-truth set are treated as hallucinated.
For consistency with CHAIR and the main table, we report F1 as the primary higher-is-better metric:
\[
    \text{F1}_i =
    \frac{2 \cdot \text{Prec}_i \cdot \text{Rec}_i}
    {\text{Prec}_i + \text{Rec}_i},
    \qquad
    \text{F1} = \frac{1}{|\mathcal{D}|}\sum_{i \in \mathcal{D}} \text{F1}_i.
\]
This provides a unified measure that rewards both avoiding hallucinated visual claims and covering ground-truth visual content.

\input{Tables/computation_setup}

\paragraph{HallusionBench.}
HallusionBench~\citep{guan2024hallusionbench} evaluates visually grounded reasoning under cases where language priors and visual evidence can conflict.
We use it as a discriminative visual-grounding benchmark.
For each candidate response, we use a GPT-based evaluator to judge whether the response is visually grounded and consistent with the reference answer.
The evaluator returns a correctness label for each response, and we treat uncertain or invalid judgments as incorrect.
The primary metric is GPT-evaluated correctness:
\[
\resizebox{1.0\columnwidth}{!}{$
    \text{Correctness}
    =
    \frac{1}{|\mathcal{D}|}
    \sum_{i \in \mathcal{D}}
    \mathbf{1}\{
        \text{Judge}(I_i, q_i, y_i) = \texttt{Correct}
    \}.
$}
\]
This single accuracy-style metric allows HallusionBench to be compared with the other benchmarks under the same Best-of-$N$ selection framework.

\subsection{Baseline Implementation Details}
\label{appendix:baseline_implementation}
For USC, when the method failed to produce a valid selection, we fell back to the first sampled response. To reduce the order bias in USC methods, we report the average accuracy over three independent runs ($m=3$).

For VAUQ, we follow the layer choices specified in the original paper: 10--25 for LLaVA-1.5-7B, 12--26 for Qwen2.5-VL-7B, and 10--25 for InternVL3-8B. We use fixed masking ratios of $K=0.6$, $0.5$, and $0.4$ for the three models, respectively, and set $\alpha=1$ throughout all experiments. We implement VAUQ following its attention-knockout-based visual masking procedure as mentioned in paper.

\subsection{Hyperparameter Selection Protocol}
\label{appendix:hp_protocol}

\name{} has two hyperparameters, $\alpha$ and $\lambda$.
We use $(\alpha,\lambda)=(7.0,1.5)$ for LLaVA-1.5-7B, $(0.5,1.25)$ for Qwen2.5-VL-7B, and
$(0.25,1.25)$ for InternVL3-8B.
These values were selected on hyperparameter search space: $\alpha \in \{0.25, 0.50, \ldots, 1.75, 2.00, 3, 4, 5, 6, 7, 8\}$ and $\lambda \in \{0.25, 0.50, \ldots, 2.75, 3.00\}$.
To validate the hyperparameter robustness of \name{}, Appendix~\ref{appendix:global_hp} reports \name{} under a single global setting shared
by all models and benchmarks, which does not depend on this selection. Also, Appendix \ref{appendix:sensitivity} reports a hyperparameter sensitivity analysis over a search space range.

The per-model values of $\alpha$ differ mainly because $\log p_t$ and $\log A_t$ lie on model-dependent numerical scales. 
LLaVA-1.5-7B uses a vocabulary of about $32$K tokens and a fixed budget of $576$ vision tokens,
whereas Qwen2.5-VL-7B and InternVL3-8B use vocabularies of about $151$K tokens and dynamic-resolution visual inputs that can exceed $2{,}000$ vision tokens.
A larger vocabulary lowers typical $\log p_t$, and a larger number of vision tokens raises typical $A_t$, so the relative scale of the two terms in Equation~\ref{eq:lookback-calibrated_token_score} differs across these models.
Accordingly, $\alpha$ absorbs this scale difference in addition to controlling the strength of
lookback calibration.

\subsection{Scoring Overhead Measurement}
\label{appendix:computation}

We measure post-generation scoring overhead in milliseconds per response.
Candidate generation is excluded from timing: all scorers are given the same generated candidate responses, and we measure only the additional wall-clock time required to assign scores after generation.
Measurements are conducted on CHAIR for LLaVA-1.5-7B, Qwen2.5-VL-7B, and InternVL3-8B.

For SC and \name{}, scoring uses model-internal quantities from the LVLM scoring pass, including token probabilities and attention weights.
For CLIPScore, we include the external CLIP encoder forward pass.
For VAUQ, we include visual-attention perturbation and uncertainty computation.
For USC, which scores all candidates from the same input in one teacher-forcing forward pass, we divide the total candidate-set scoring time by the number of candidates.
We report mean milliseconds per response with standard deviation across evaluated samples.

Due to resource availability, LLaVA-1.5-7B is measured on an NVIDIA RTX A6000, while Qwen2.5-VL-7B and InternVL3-8B are measured on an NVIDIA H200.
Therefore, absolute latency values should be interpreted as scorer-overhead comparisons within each model, rather than as direct speed comparisons across LVLMs.

\input{Figures/appendix_pos}

\newpage

\section{Additional Quantitative Results}
\label{appendix:quanti}

\subsection{PoS}
\label{appendix:pos}

In Section~\ref{sec:motivation_2}, we group generated words into a visual-reference-prone content set and a textual-function set to compare visual lookback score with output-space confidence.
Here, we provide a fine-grained breakdown by individual POS tags.

We use the same generated responses and scoring setup as in Section~\ref{sec:motivation_2}.
Each word is assigned a POS tag using spaCy.
For words split into multiple subtokens, we average the corresponding subtoken scores to obtain a word-level score.
Visual lookback score and Self-Certainty are independently standardized over all analyzed output words within each model--benchmark setup, and we report the mean z-score for each POS category.

Figure~\ref{fig:appendix_pos} shows that content-oriented POS categories, such as nouns, proper nouns, adjectives, and numerals, generally exhibit higher visual lookback score than textual-function categories.
In contrast, textual-function categories often show relatively higher certainty, reflecting their linguistic predictability.
This fine-grained analysis supports the coarse POS grouping used in the main text.
At the same time, some rare or context-dependent categories exhibit noisier behavior, which motivates our use of grouped POS sets rather than drawing conclusions from individual POS tags.

\input{Tables/random_image_agreement}

\subsection{Random-Image Replacement Control}
\label{appendix:random_image}

\input{Tables/random_image_control}

\input{Tables/global_hyperparameters}

Section~\ref{sec:motivation_1} compares scoring with and without the input image.
Removing the image changes the format of the scoring context, so we
additionally run a control in which the context remains a well-formed image--text input and only
the visual evidence is mismatched.
For each instance we generate the $N=25$ candidates under the original image, freeze the candidate
set, and then replace only the conditioning image at scoring time with an image sampled uniformly
at random from the same benchmark.
Because the candidates never change, differences in selection are attributable to the scoring
function rather than to candidate quality or to dataset-specific answer priors.
We note that this control fixes the candidates after generation, so it isolates the scoring
function and does not speak to the generator itself.
All numbers use LLaVA-1.5-7B with $N=25$ on VQAv2 and CHAIR.

Table~\ref{tab:appendix_random_image} reports the top-1 agreement ratio between original-image and
random-image scoring.
SC selects the same top-1 candidate for $0.594$ of instances on average, far above $1/N=0.04$,
confirming the observation in Section~\ref{sec:motivation_1} under a matched input format.
\name{} changes its selection far more often ($0.278$), and the visual lookback score alone falls
in between.

Table~\ref{tab:random_image} reports the same control at the performance level, where the removed
gain defined as
\[
\frac{\text{Acc}^{\text{orig-img}}-\text{Acc}^{\text{rand-img}}}
     {\text{Acc}^{\text{orig-img}}-\text{Acc}^{\text{rand. selection}}} .
\]
A value above $100\%$ means that random-image scoring falls below random selection, which is the
case for \name{} on VQAv2.
\name{} thus loses most of its advantage when the visual reference is mismatched, whereas SC
retains roughly half of its gain on VQAv2 and is unchanged on CHAIR.

\subsection{Robustness to Hyperparameters}
\label{appendix:global_hp}

The hyperparameters used in the main results are model-specific
(Appendix~\ref{appendix:hp_protocol}).
To test whether the conclusions depend on this choice, we evaluate a single global setting
$(\alpha,\lambda)=(0.25,1.25)$ for every model and benchmark, which removes model-wise scaling
entirely.
Table~\ref{tab:appendix_global_hp} reports model-wise averages over four benchmarks and two
candidate budgets, $N=5$ and $N=25$.
Under this global setting \name{} remains the best method on average and drops $0.19$ points
relative to the reported configuration.

\subsection{Sensitivity Range}
\label{appendix:sensitivity}

Table~\ref{tab:appendix_hp_range} sweeps hyperparameter search space mentioned on \ref{appendix:hp_protocol} and summarizes the resulting distribution of Best-of-$N$ performance.
The reported configuration lies inside the observed range in all $24$ model--benchmark--$N$ settings and within two standard deviations of the range mean in $23$ of $24$ settings.
The range mean alone already exceeds every competing baseline in many settings (underlined), and in two settings it exceeds the reported \name{} value (bold).

\input{Tables/hyperparameter_sensitivity}

\input{Tables/multi_seed_results}

\input{Tables/lvlm_judge}

\subsection{Variance across Random Seeds}
\label{appendix:seed}

Candidates are sampled stochastically, so single-run numbers carry sampling variance.
We repeat the $N=5$ setting with two additional seeds, $\{123, 456\}$, and report mean $\pm$ std
over three runs including the original one.
We use VQAv2 and CHAIR because together they cover short-answer VQA and open-ended hallucination
evaluation while keeping the additional computation manageable.
As shown in Table~\ref{tab:appendix_seed}, \name{} obtains the highest mean in all four
model--benchmark pairs, while the strongest competing baseline changes across settings.
With three seeds we report means and standard deviations rather than significance tests.

\subsection{Comparison with Pairwise LVLM Judge}
\label{appendix:judge}

USC selects a response by prompting the LVLM with all candidates at once.
Following the observation that pairwise comparison is closer to human judgment than batch comparison~\citep{chen2024mllmjudge}, we implement a stronger self-judging variant: a pairwise
tournament over the $N=5$ candidates using LLaVA-1.5-7B, which requires four pairwise comparisons
rather than all $\binom{5}{2}$ pairs.
Table~\ref{tab:appendix_judge} shows that the pairwise judge improves over USC on average
($61.21$ versus $60.64$), and that \name{} obtains the best score on every benchmark, with a
$+2.13$ point improvement in the overall average.

\input{Tables/human_evaluation}

\input{Figures/human_evaluation_instructions}

\subsection{Human Evaluation}
\label{appendix:human}

We compare \name{} against SC by human judgment on CHAIR with LLaVA-1.5-7B and $N=25$.
Among $1{,}000$ instances the two methods select different top-ranked responses in $779$ cases, and
we evaluate these disagreement cases since identical selections carry no comparative information.
We randomly sample $50$ disagreement cases and shuffle the presentation order of the two responses.
The samples are split into two disjoint sets of $25$, each evaluated by six anonymous participants
with no participant overlap between the sets, yielding $300$ individual judgments.
As shown in Figure~\ref{fig:human_eval_instruction}, annotators were shown the image and two responses with the evaluation instructions, and were asked which response better described the image, with a tie option.

Table~\ref{tab:appendix_human} summarizes the outcome.
Automatic CHAIR-F1 declares a tie in $24$ of $50$ samples ($48.0\%$) and splits the remainder
evenly.
Human annotators prefer the \name{}-selected response in $151$ of $300$ individual judgments
($50.3\%$) versus $105$ ($35.0\%$) for SC, and, after aggregating non-tie votes within each sample,
on $27$ samples ($54.0\%$) versus $14$ ($28.0\%$).
\name{} can thus select responses that humans prefer in cases where the automatic metric is
indifferent.

\clearpage
\section{Additional Qualitative Results}
\label{appendix:quali}

In Figures~\ref{fig:appendix_quali_1} and \ref{fig:appendix_quali_5}, we additionally present the qualitative examples to compare SC, Visual lookback score, and our final score. 

Figure~\ref{fig:appendix_quali_1} visualizes token-level scores for two representative cases in which \name{} selects the correct top-ranked response, while both SC and Visual Lookback Score select incorrect responses.
These examples illustrate why neither confidence nor visual lookback alone is sufficient.
SC often assigns large scores to a small number of highly confident tokens, but these tokens need not correspond to visually grounded evidence for the question.
In contrast, Visual lookback Score broadly highlights tokens that refer to visible entities in the image, but it can also emphasize irrelevant or distracting visual content.
For example, in the counting case, Visual lookback Score attends to image-referential words such as \textit{cars}, \textit{house}, and \textit{truck}, but fails to distinguish the question-relevant evidence needed to answer the number of parked cars.
\name{} instead emphasizes tokens that are both visually grounded and relevant to the question, such as the single car in the image, leading to the correct selection.
A similar pattern appears in Figure~\ref{fig:appendix_quali_5}: Visual lookback score highlights visible entities such as the child, toothbrush, and person, while \name{} assigns high scores to tokens supporting the question-relevant relation between the two people.
These qualitative results support our design choice of combining token-level confidence with visual lookback and aggregating the resulting scores according to visual relevance.

\subsection{Failure Case Analysis}
\label{appendix:failure}
Figure~\ref{fig:appendix_quali_failure} shows an instance with LLaVA-1.5-7B on VQAv2, where the
question asks whether the stuffed animals are bigger than the baby.
The response selected by \name{} places high visual lookback scores on both compared entities and
on the comparative token itself, so the response is visually referential precisely at the positions
the question depends on.
The comparison is nevertheless resolved incorrectly.

This delimits what the visual lookback score measures.
A high $A_t$ indicates that the prediction step for $y_t$ consulted the vision tokens; it does not
indicate that the model resolved the relation among the entities it referenced.
Visual reference usage is therefore informative when the answer depends on which entities are
present in the image, and far less informative when the entities are jointly visible and the answer
depends only on how they compare.
Questions of the latter type are frequent in HallusionBench, which is consistent with \name{} being
least effective on that benchmark.
USC compares candidates using the LVLM itself and can exploit its comparative reasoning, while
CLIPScore uses a global image--text alignment signal that is less sensitive to which tokens
attended where, so both can recover the correct response in this regime.

\section{Usage of AI assistants}
In preparing this work, we utilized AI-based writing assistants in a limited way to suggest alternative phrasings, correct grammatical errors, and improve the readability.
All research ideas, experimental designs, implementation details, and reported results were determined, implemented, and verified by the authors, and all model outputs and quantitative results in the paper were generated and checked through our own code and experiments.
AI assistants did not generate any factual content reported in the paper (e.g., experimental results, dataset statistics, or citations), ensuring the originality, scientific contributions, technical content, methodology, and experimental findings are entirely attributable to the authors.

\input{Figures/qualitative_1}

\input{Figures/qualitative_2}

\input{Figures/qualitative_failure}

%% file: Tables/oracle_results.tex
\begin{table}[t]

\centering
\setlength{\tabcolsep}{4pt}
\caption{\textbf{Oracle Best-of-$N$ performance.}
Oracle selection performance for $N=5$ and $N=25$ across four benchmarks and three LVLMs.
Average denotes the average score across the four benchmarks for each candidate budget.}
\vspace{0.05in}
\label{tab:oracle_by_benchmark}
\begin{small}
\resizebox{1.0\linewidth}{!}{
\begin{tabular}{l|cc|cc|cc}
    \toprule
    \multirow{2}{*}{\textbf{Benchmark}}
    & \multicolumn{2}{c|}{\textbf{LLaVA-1.5-7B}}
    & \multicolumn{2}{c|}{\textbf{Qwen2.5-VL-7B}}
    & \multicolumn{2}{c}{\textbf{InternVL3-8B}} \\
    \cmidrule(lr){2-3}\cmidrule(lr){4-5}\cmidrule(lr){6-7}
    & $N=5$ & $N=25$ & $N=5$ & $N=25$ & $N=5$ & $N=25$ \\
    \midrule
    VQAv2
    & 84.07 & 92.50
    & 82.83 & 91.17
    & 88.60 & 94.80 \\

    CHAIR
    & 84.17 & 90.55
    & 86.11 & 91.47
    & 86.73 & 91.72 \\

    HallusionBench
    & 80.71 & 94.42
    & 83.42 & 92.89
    & 80.54 & 91.71 \\

    AMBER
    & 76.96 & 83.83
    & 84.97 & 89.91
    & 88.81 & 92.76 \\
    \midrule
    \textbf{Average}
    & 81.48 & 90.33
    & 84.33 & 91.36
    & 86.17 & 92.75 \\
    \bottomrule
\end{tabular}
}
\end{small}
\end{table}

%% file: Tables/computation_setup.tex
\begin{table*}[t]
    \centering
    \small
    \caption{\textbf{Hardware and software configuration.}
    Environment used for the scoring overhead measurement in Figure~\ref{fig:computation}, reported separately for the two GPU setups.}
    \label{tab:computation_setup}
    \begin{tabular}{lcc}
        \toprule
         & LLaVA-1.5-7B & Qwen2.5-VL-7B / InternVL3-8B \\
        \midrule
        GPU & NVIDIA RTX A6000 & NVIDIA H200 \\
        CPU & AMD EPYC 9354 & Intel Xeon Platinum 8568Y+ \\
        CUDA & 11.7 & 12.4 \\
        cuDNN & 8.5.0 & 9.1.0 \\
        Python & 3.10.15 & 3.9.21 \\
        PyTorch & 2.0.1+cu117 & 2.6.0+cu124 \\
        Transformers & 4.31.0 & 4.57.6 \\
        Random seed & 42 & 42 \\
        \bottomrule
    \end{tabular}
\end{table*}

%% file: Figures/appendix_pos.tex
\begin{figure*}[t]
    \includegraphics[width=\linewidth]{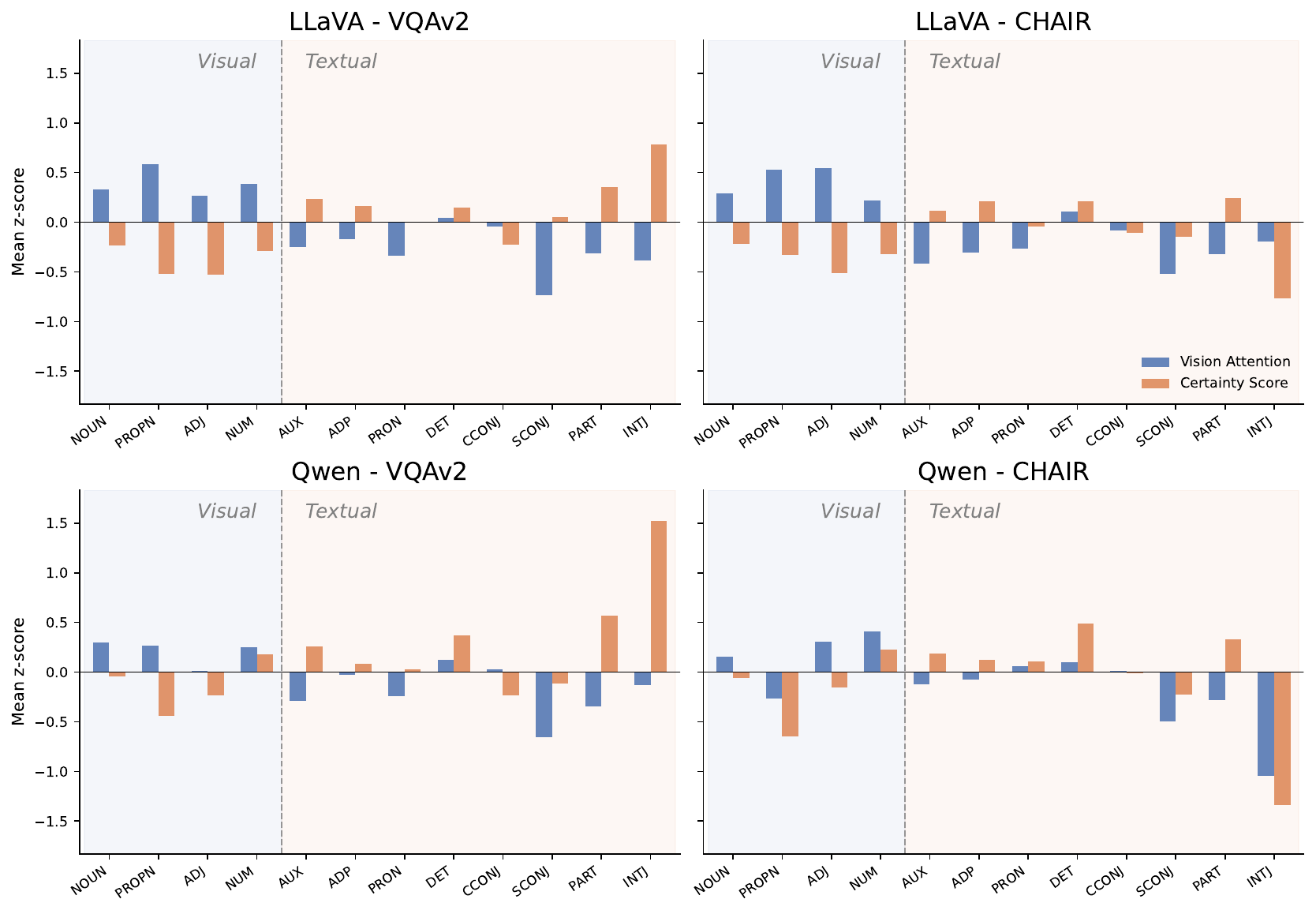}
    \vspace{-0.25in}
    \caption{\textbf{Visual lookback and SC emphasize different POS groups.}
    We report mean z-scored Visual lookback and Self-Certainty (SC) scores for each POS category.
    Across models and benchmarks, visual POS categories tend to receive higher Visual lookback scores, whereas textual categories tend to receive higher SC scores.
    }
    \vspace{-0.1in}
    \label{fig:appendix_pos}
\end{figure*}

%% file: Tables/random_image_agreement.tex
\begin{table}[t]
\centering
\small
\caption{\textbf{Top-1 agreement between original-image and random-image scoring.}
Fraction of instances for which the two scorings select the same top-ranked candidate, with LLaVA-1.5-7B and $N=25$.
Lower values indicate that the selected candidate changes more when the visual reference is
replaced.}
\label{tab:appendix_random_image}
\resizebox{\columnwidth}{!}{
\begin{tabular}{lccc}
    \toprule
    \textbf{Benchmark} & \textbf{SC} & \textbf{Visual Lookback} & \textbf{\name{}} \\
    \midrule
    VQAv2 & 0.550 & 0.420 & \textbf{0.288} \\
    CHAIR & 0.637 & 0.312 & \textbf{0.268} \\
    \midrule
    Avg.  & 0.594 & 0.366 & \textbf{0.278} \\
    \bottomrule
\end{tabular}
}
\end{table}

%% file: Tables/random_image_control.tex
\begin{table*}[t]

\centering
\small
\caption{\textbf{Random-image replacement control.}
Best-of-$N$ selection accuracy with LLaVA-1.5-7B ($N=25$) when the candidate set is fixed and only the image used for \emph{scoring} is replaced by a randomly sampled image.
\emph{Removed} denotes the fraction of each scorer's gain over random selection that disappears under replacement; values above 100\% indicate that random-image scoring falls below random selection.}
\label{tab:random_image}
\begin{tabular}{lccccc}
    \toprule
    \multirow{2}{*}[-0.5ex]{\textbf{Benchmark}}
    & \multirow{2}{*}[-0.5ex]{\textbf{Random}}
    & \multicolumn{2}{c}{\textbf{SC}}
    & \multicolumn{2}{c}{\textbf{\name{}}} \\
    \cmidrule(lr){3-4}\cmidrule(lr){5-6}
    & & Orig.$\rightarrow$Rand. & Removed & Orig.$\rightarrow$Rand. & Removed \\
    \midrule
    VQAv2 & 60.40 & 65.70 $\rightarrow$ 63.10 & 49.1\% & 67.60 $\rightarrow$ 59.50 & \textbf{112.5\%} \\
    CHAIR & 71.52 & 73.53 $\rightarrow$ 73.64 & $-5.5$\% & 74.43 $\rightarrow$ 71.81 & \textbf{90.0\%} \\
    \bottomrule
\end{tabular}
\end{table*}

%% file: Tables/global_hyperparameters.tex
\begin{table*}[h]
\centering
\small
\caption{\textbf{Single global hyperparameter setting.}
Model-wise averages over four benchmarks and two candidate budgets ($N=5,25$) under
$(\alpha,\lambda)=(0.25,1.25)$ shared across all settings.}
\label{tab:appendix_global_hp}
\begin{tabular}{lcccc}
    \toprule
    \textbf{Method} & \textbf{LLaVA} & \textbf{Qwen} & \textbf{InternVL} & \textbf{Avg.} \\
    \midrule
    Random    & 60.16 & 66.95 & 69.00 & 65.37 \\
    SC        & 62.36 & 67.78 & 70.81 & 66.98 \\
    USC       & 60.28 & 68.11 & 71.58 & 66.66 \\
    CLIPScore & 61.03 & 67.45 & 70.25 & 66.24 \\
    VAUQ      & 62.13 & 68.17 & 71.53 & 67.28 \\
    \name{} (global HP)   & \textbf{63.09} & \textbf{69.91} & \textbf{72.27} & \textbf{68.43} \\
    \midrule
    \name{} (reported HP) & 63.67 & 69.91 & 72.27 & 68.62 \\
    \bottomrule
\end{tabular}
\end{table*}

%% file: Tables/hyperparameter_sensitivity.tex
\begin{table*}[t]
\centering
\small
\caption{\textbf{Hyperparameter sensitivity over a broad range.}
Each entry reports $N=5$\,/\,$N=25$ over $\alpha \in \{0.25,0.50,\ldots,1.75,2.00,3,4,5,6,7,8\}$ and $\lambda \in \{0.25,0.50,\ldots,2.75,3.00\}$.
In the mean column, \underline{underlined} values indicate that the range mean exceeds all competing baselines for the corresponding setup, and \textbf{bold} values indicate that the range mean exceeds the reported \name{} value.}
\label{tab:appendix_hp_range}
\resizebox{\linewidth}{!}{
\begin{tabular}{llcccc}
    \toprule
    \textbf{Model} & \textbf{Benchmark} & \textbf{Min} & \textbf{Max} & \textbf{Mean $\pm$ Std.} & \textbf{Reported \name{}} \\
    \midrule
    \multirow{4}{*}{LLaVA-1.5-7B}
      & VQAv2          & 65.03\,/\,65.87 & 67.30\,/\,69.07 & 66.03 $\pm$ 0.46\,/\,\underline{67.49} $\pm$ 0.85 & 66.63\,/\,67.60 \\
      & CHAIR          & 72.99\,/\,73.36 & 74.17\,/\,74.99 & \underline{73.60} $\pm$ 0.27\,/\,\textbf{74.47} $\pm$ 0.28 & 74.03\,/\,74.43 \\
      & HallusionBench & 45.18\,/\,45.35 & 48.05\,/\,48.73 & 46.63 $\pm$ 0.61\,/\,\underline{46.49} $\pm$ 0.79 & 47.38\,/\,48.22 \\
      & AMBER          & 64.48\,/\,64.25 & 65.68\,/\,65.84 & 65.12 $\pm$ 0.27\,/\,65.06 $\pm$ 0.41 & 65.33\,/\,65.74 \\
    \midrule
    \multirow{4}{*}{Qwen2.5-VL-7B}
      & VQAv2          & 66.57\,/\,63.80 & 68.93\,/\,69.10 & \underline{67.82} $\pm$ 0.55\,/\,\underline{67.40} $\pm$ 0.98 & 68.23\,/\,67.47 \\
      & CHAIR          & 73.72\,/\,71.52 & 75.63\,/\,75.95 & 75.09 $\pm$ 0.38\,/\,73.91 $\pm$ 1.14 & 75.42\,/\,75.28 \\
      & HallusionBench & 56.18\,/\,58.71 & 60.91\,/\,62.77 & \underline{59.03} $\pm$ 0.93\,/\,\underline{60.46} $\pm$ 0.82 & 60.07\,/\,61.93 \\
      & AMBER          & 74.46\,/\,73.56 & 75.26\,/\,76.16 & 74.96 $\pm$ 0.17\,/\,\underline{75.05} $\pm$ 0.61 & 74.98\,/\,75.92 \\
    \midrule
    \multirow{4}{*}{InternVL3-8B}
      & VQAv2          & 66.70\,/\,66.63 & 70.97\,/\,73.43 & 69.29 $\pm$ 0.93\,/\,\underline{71.14} $\pm$ 1.63 & 70.33\,/\,72.57 \\
      & CHAIR          & 78.04\,/\,78.04 & 79.40\,/\,78.56 & \underline{78.74} $\pm$ 0.37\,/\,\textbf{78.31} $\pm$ 0.11 & 79.37\,/\,78.22 \\
      & HallusionBench & 55.16\,/\,54.82 & 57.36\,/\,58.71 & 56.15 $\pm$ 0.49\,/\,56.34 $\pm$ 0.90 & 57.02\,/\,56.85 \\
      & AMBER          & 80.48\,/\,80.13 & 81.87\,/\,82.22 & 81.49 $\pm$ 0.26\,/\,81.65 $\pm$ 0.40 & 81.77\,/\,82.07 \\
    \bottomrule
\end{tabular}
}
\end{table*}

%% file: Tables/multi_seed_results.tex
\begin{table*}[t]
\centering
\small
\setlength{\tabcolsep}{3pt}
\caption{\textbf{Multi-seed results.} Best-of-$N$ selection performance at $N=5$, reported as mean $\pm$ standard deviation over seeds $\{42,123,456\}$.}
\label{tab:appendix_seed}
\begin{tabular}{lcccc}
    \toprule
    \multirow{2}{*}[-0.5ex]{\textbf{Method}}
    & \multicolumn{2}{c}{\textbf{LLaVA-1.5-7B}}
    & \multicolumn{2}{c}{\textbf{Qwen2.5-VL-7B}} \\
    \cmidrule(lr){2-3}\cmidrule(lr){4-5}
    & VQAv2 & CHAIR & VQAv2 & CHAIR \\
    \midrule
    Random    & 62.80 $\pm$ 0.27 & 74.49 $\pm$ 3.73 & 64.56 $\pm$ 1.01 & 74.36 $\pm$ 0.45 \\
    SC        & 66.61 $\pm$ 0.98 & 75.66 $\pm$ 2.52 & 66.86 $\pm$ 0.76 & 75.05 $\pm$ 1.16 \\
    USC       & 64.93 $\pm$ 1.00 & 75.15 $\pm$ 3.54 & 67.99 $\pm$ 1.10 & 75.20 $\pm$ 0.59 \\
    CLIPScore & 63.72 $\pm$ 1.16 & 75.57 $\pm$ 3.20 & 64.65 $\pm$ 0.66 & 74.82 $\pm$ 0.31 \\
    VAUQ      & 66.14 $\pm$ 0.16 & 75.91 $\pm$ 2.98 & 67.00 $\pm$ 0.42 & 75.53 $\pm$ 0.92 \\
    \name{}   & \textbf{66.80 $\pm$ 0.67} & \textbf{76.38 $\pm$ 2.03} & \textbf{68.28 $\pm$ 0.05} & \textbf{76.09 $\pm$ 0.97} \\
    \bottomrule
\end{tabular}
\end{table*}

%% file: Tables/lvlm_judge.tex
\begin{table*}[t]
\centering
\small
\setlength{\tabcolsep}{4pt}
\caption{\textbf{Pairwise LVLM-as-a-judge comparison.} Best-of-$N$ selection performance with LLaVA-1.5-7B at $N=5$ across four benchmarks.}
\label{tab:appendix_judge}
\begin{tabular}{lccccc}
    \toprule
    \textbf{Method} & \textbf{VQAv2} & \textbf{CHAIR} & \textbf{HallusionBench} & \textbf{AMBER} & \textbf{Avg.} \\
    \midrule
    Random          & 62.57 & 70.20 & 44.84 & 63.90 & 60.38 \\
    USC             & 63.78 & 71.07 & 43.09 & 64.61 & 60.64 \\
    LVLM-as-a-Judge & 64.90 & 70.58 & 44.67 & 64.67 & 61.21 \\
    \name{}         & \textbf{66.63} & \textbf{74.03} & \textbf{47.38} & \textbf{65.33} & \textbf{63.34} \\
    \bottomrule
\end{tabular}
\end{table*}

%% file: Tables/human_evaluation.tex
\begin{table*}[h]
\centering
\small
\caption{\textbf{Human evaluation on CHAIR disagreement cases.}
Preference between \name{}- and SC-selected responses on 50 sampled disagreement cases with LLaVA-1.5-7B ($N=25$), collected from six annotators for 300 judgments in total. Automatic denotes the CHAIR-F1 verdict on the same samples.}
\label{tab:appendix_human}
\begin{tabular}{lccc}
    \toprule
    \textbf{Evaluation} & \textbf{\name{}} & \textbf{Tie} & \textbf{SC} \\
    \midrule
    Automatic (CHAIR-F1) & 13 (26.0\%)  & 24 (48.0\%) & 13 (26.0\%) \\
    Human (per judgment) & \textbf{151 (50.3\%)} & 44 (14.7\%) & 105 (35.0\%) \\
    Human (per sample) & \textbf{27 (54.0\%)} & 9 (18.0\%) & 14 (28.0\%) \\
    \bottomrule
\end{tabular}
\end{table*}

%% file: Figures/human_evaluation_instructions.tex
\begin{figure*}[!t]
\centering

\begin{tcolorbox}[
    width=0.96\textwidth,
    colback=white,
    colframe=black!70,
    boxrule=0.8pt,
    arc=3mm,
    left=4mm,
    right=4mm,
    top=3mm,
    bottom=3mm
]
\small\ttfamily

You will see 25 pairs of an image, a question, and two
model-generated image descriptions, labeled A1 and A2.

\medskip

Your task is to choose one response based on the following criterion:

``Which response includes more accurate visual elements or
descriptions of the given image''

\medskip

If both descriptions contain some incorrect or hallucinated
details, do not automatically choose Tie. Instead, choose the
description that is overall more visually faithful.

\medskip
\medskip

\textbf{Choose}:

A1, if A1 is more visually faithful than A2.

A2, if A2 is more visually faithful than A1.

Tie, if both descriptions are similarly faithful or similarly
unfaithful.

\medskip
\medskip

The purpose of this study is to collect data for academic and
research analysis only.

\medskip

Your participation is anonymous, and all responses will be used
exclusively for research purposes. The collected data will be
handled responsibly and will not be misused.

\end{tcolorbox}

\caption{
\textbf{Instructions provided to participants for the human evaluation.} Participants were asked to compare two model-generated descriptions and select the more visually faithful one to the given image, with a tie option.
}
\label{fig:human_eval_instruction}
\end{figure*}

%% file: Figures/qualitative_1.tex
\begin{figure*}[t]
    \includegraphics[width=\linewidth]{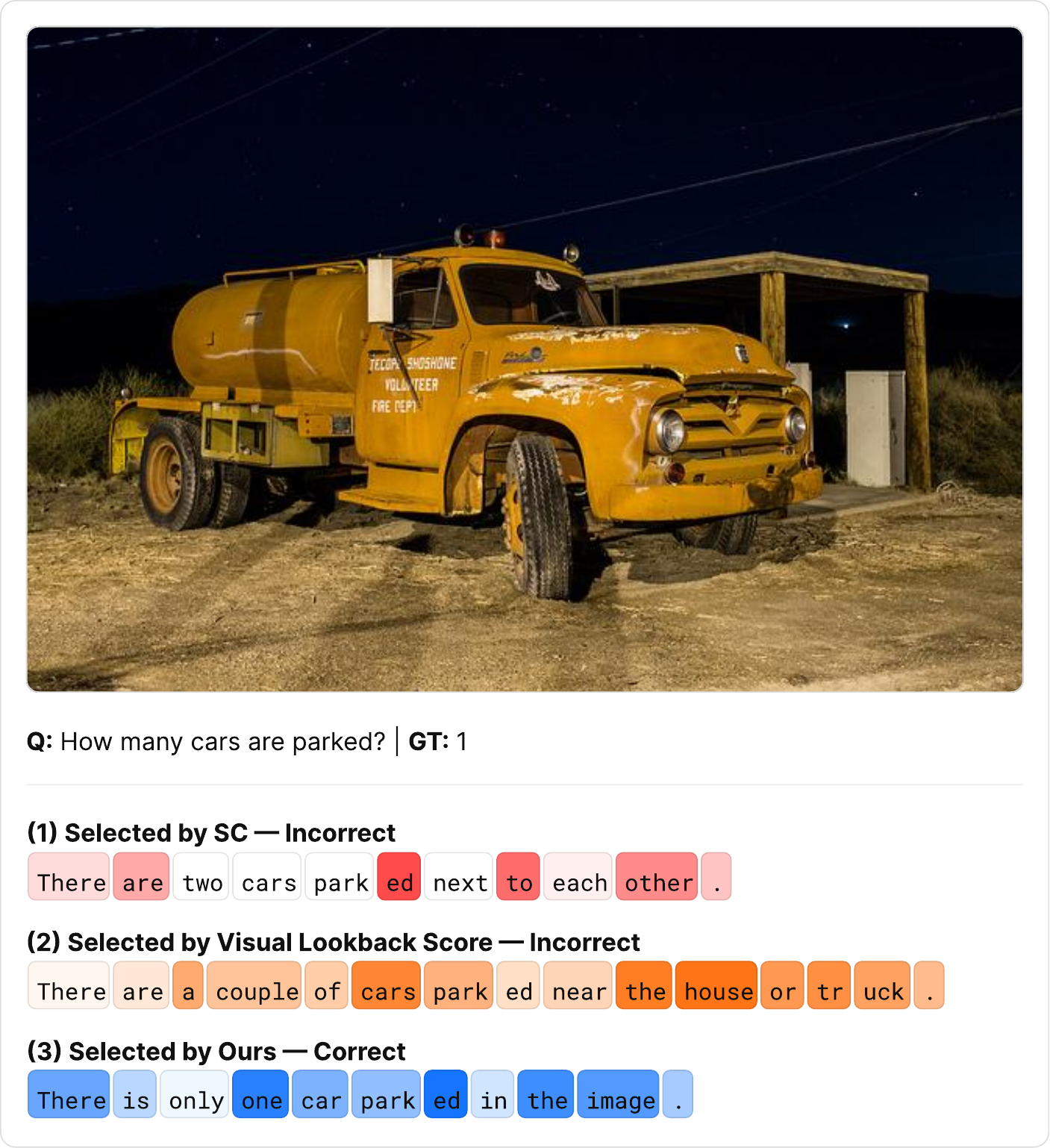}
    \caption{
    \textbf{Token-level comparison of response scorers.}
    SC favors a fluent but incorrect response, and Visual Lookback Score highlights visually referential tokens but remains insufficient on its own.
    \name{} combines confidence and visual lookback to select the visually grounded response.
    Darker highlights indicate higher token-level scores.
    }
    \label{fig:appendix_quali_1}
\end{figure*}

%% file: Figures/qualitative_2.tex
\begin{figure*}[t]
    \includegraphics[width=\linewidth]{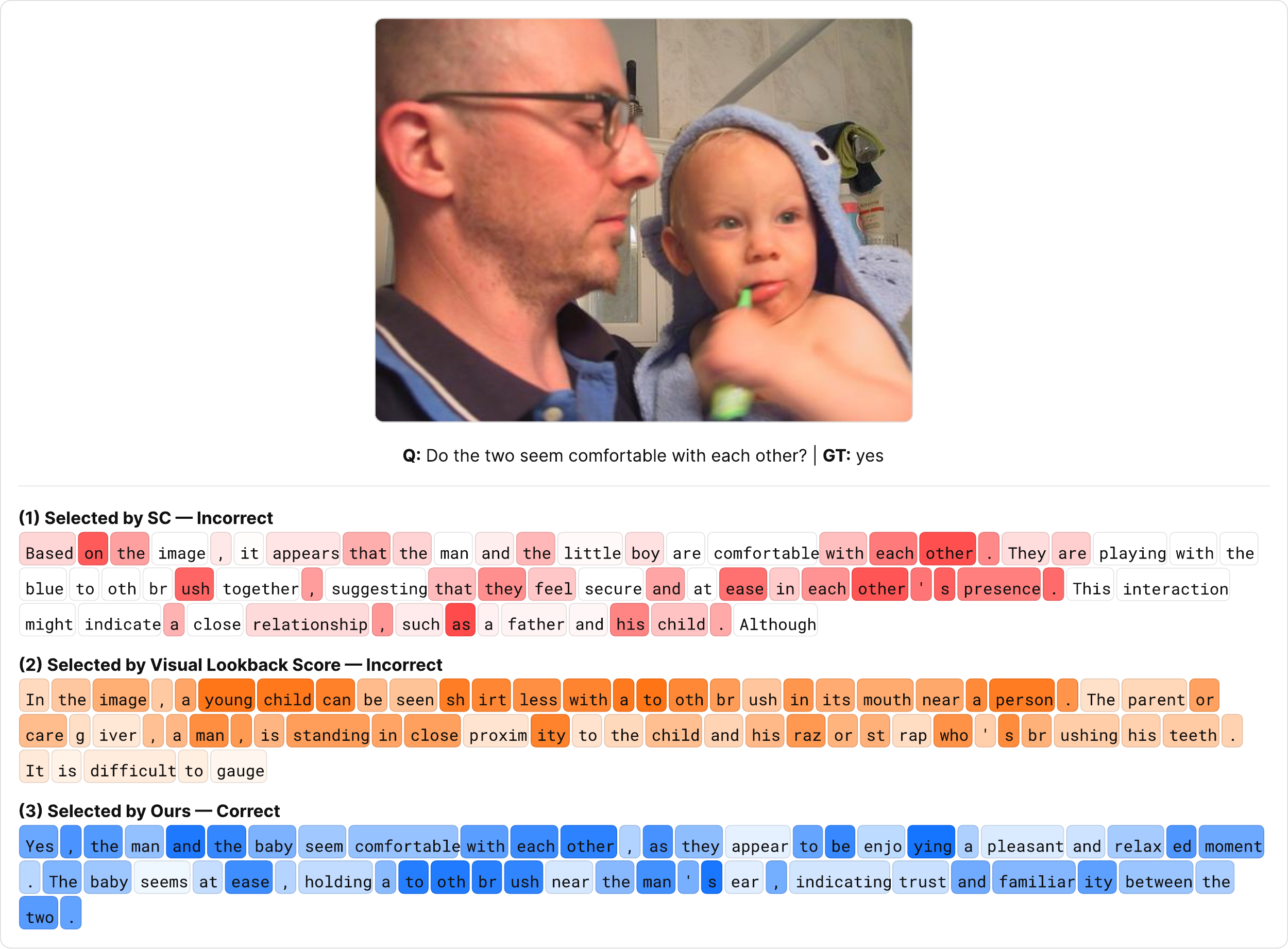}
    \caption{
    \textbf{Token-level behavior on a visually complex example.}
    Visual lookback Score broadly highlights image-referential tokens, whereas SC places large scores on a few confident tokens.
    \name{} instead emphasizes question-relevant visual evidence, selecting the correct response.
    Darker highlights indicate higher token-level scores.
    }
    \label{fig:appendix_quali_5}
\end{figure*}

%% file: Figures/qualitative_failure.tex
\begin{figure*}[t]
\centering
\includegraphics[width=\linewidth]{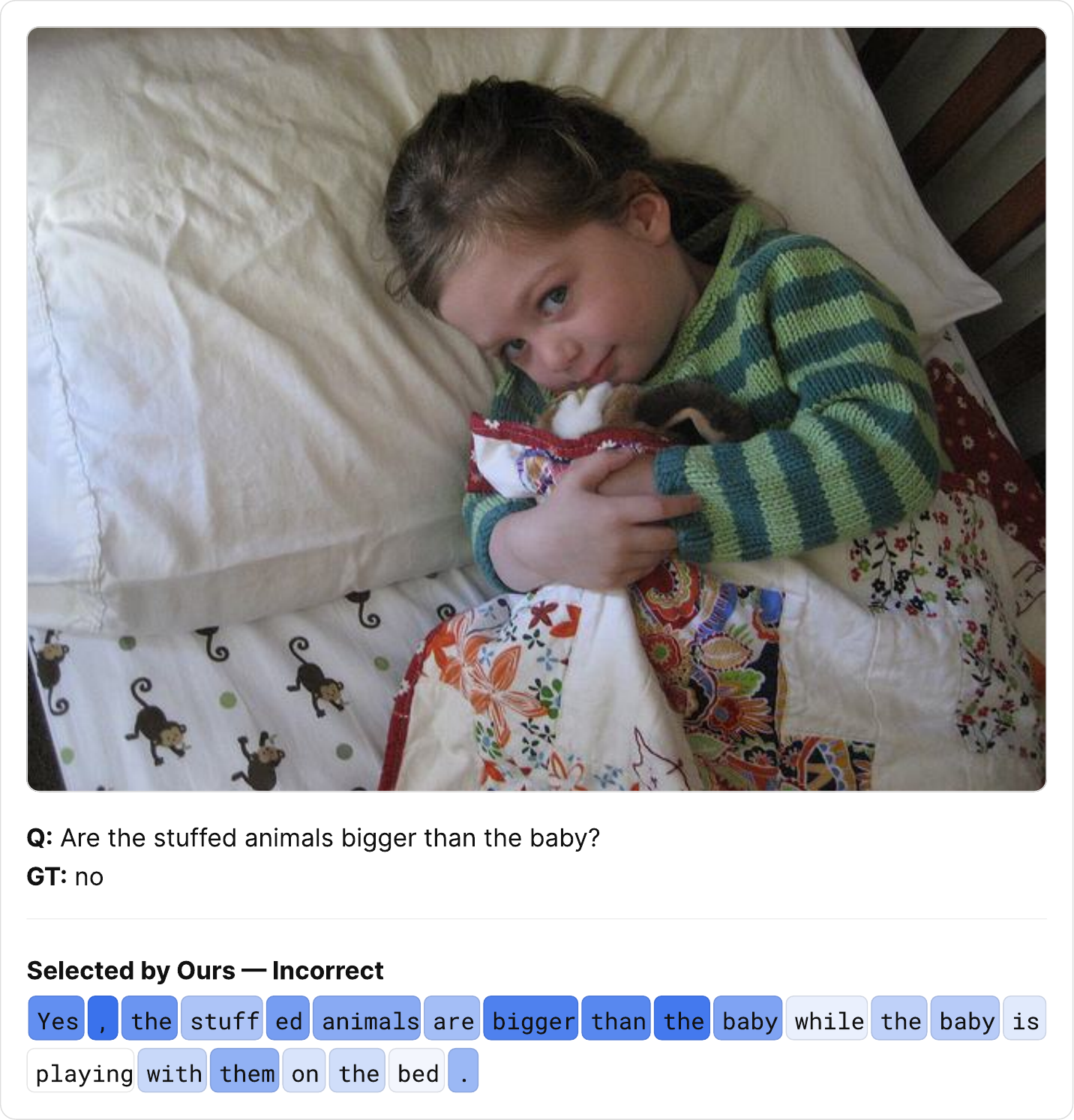}
\caption{{\textbf{Failure case of \name{}}.
The response selected by \name{} attends strongly to both entities being compared and to the comparative expression, but still reaches the wrong conclusion. This shows that a high visual lookback score indicates that the model consulted the image during generation, not that it correctly understood the relationship between the referenced entities.
Darker highlights indicate higher token-level \name{} scores.}}
\label{fig:appendix_quali_failure}
\end{figure*}